\documentclass[letterpaper, 10 pt, conference]{ieeeconf}  

\IEEEoverridecommandlockouts                              

\usepackage[english]{babel}

\makeatletter
\adddialect\l@ENGLISH\l@english
\makeatother

\usepackage{cite}

\usepackage{epsfig}
\usepackage{pst-all, graphics, graphicx, color}
\usepackage[crop=pdfcrop]{pstool}
\usepackage{psfrag}
\usepackage{subfigure}
\usepackage{amsmath, amssymb} 
\usepackage{epstopdf}
\usepackage{tikz}
\usepackage{multirow}
\usepackage{float}
\usepackage{algorithm, algorithmic}
\usepackage{dsfont}

\usepackage[makeroom]{cancel}

\usepackage{pifont}
\newcommand{\cmark}{\ding{51}}%

\floatstyle{ruled}
\newfloat{algorithm}{tbp}{loa}
\providecommand{\algorithmname}{Algorithm}
\floatname{algorithm}{\protect\algorithmname}

\makeatletter
\newcommand\footnoteref[1]{\protected@xdef\@thefnmark{\ref{#1}}\@footnotemark}
\makeatother

\renewcommand{\algorithmiccomment}[1]{\bgroup\hfill\scriptsize//~#1\egroup}

\def\tr{^{\rm T}}

\def\zero{\hbox{\bf 0}}

\def\bff{{\mbox{\boldmath $f$}}}
\def\bfg{{\mbox{\boldmath $g$}}}

\def\bfx{{\mbox{\boldmath $x$}}}
\def\bfy{{\mbox{\boldmath $y$}}}
\def\bfz{{\mbox{\boldmath $z$}}}

\def\bfX{{\mbox{\boldmath $X$}}}
\def\bfY{{\mbox{\boldmath $Y$}}}

\def\bfPhi{{\mbox{\boldmath $\Phi$}}}

\newtheorem{theorem}{\bf Theorem}

\newtheorem{remark}{Remark}

\begin{document}
\bstctlcite{IEEEexample:BSTcontrol}

\title{\LARGE \bf
{An Energy-based Approach to Ensure the Stability of Learned Dynamical Systems}}

\author{Matteo Saveriano%
\thanks{The author is with Intelligent and Interactive Systems and Digital Science Center (DiSC), University of Innsbruck, Innsbruck, Austria {\tt matteo.saveriano@uibk.ac.at}.}
\thanks{This research has received funding from the European Union's Horizon 2020 research and innovation programme (grant agreement no. 731761, IMAGINE) and from the Austrian Research Foundation (Euregio IPN 86-N30, OLIVER).}
}


\maketitle


\begin{abstract}
Non-linear dynamical systems represent a compact, flexible, and robust tool for reactive motion generation. The effectiveness of dynamical systems relies on their ability to accurately represent stable motions. Several approaches have been proposed to learn stable and accurate motions from demonstration. Some approaches work by separating accuracy and stability into two learning problems, which increases the number of open parameters and the overall training time. Alternative solutions exploit single-step learning but restrict the applicability to one regression technique. This paper presents a single-step approach to learn stable and accurate motions that work with any regression technique. The approach makes energy considerations on the learned dynamics to stabilize the system at run-time while introducing small deviations from the demonstrated motion. Since the initial value of the energy injected into the system affects the reproduction accuracy, it is estimated from training data using an efficient procedure. Experiments on a real robot and a comparison on a public benchmark shows the effectiveness of the proposed approach.
\end{abstract}


\IEEEpeerreviewmaketitle

\section{Introduction}\label{sec:intro}
Service robots need to perform a variety of tasks in different domains, which requires flexible and fast motion planners. Stable dynamical systems (DS) arise as a promising approach for real-time motion generation and gained increasing attention in the community \cite{DMP, Pervez18, SEDS, Clf, tau-SEDS, Perrin16, Blocher17, duan2017fast, NeuralLearn2}. DS have several interesting properties that make them well-suited for motion generation. Indeed, DS can plan stable motions from any starting point to any target position in the robot's workspace. DS can on-line replan the robot's trajectory to cope with changes in the target position or unexpected obstacles \cite{DS_avoidance, Saveriano13, Saveriano14, Saveriano17}. DS motions can be incrementally updated as novel task demonstrations are provided \cite{Kronander15, Saveriano18} or confined to constrained domains \cite{Saveriano19_Learning} without compromising the convergence to the goal. 
Finally, DS have been used to encode primitive robotic skills or movement primitives \cite{Schaal99}, which can be properly scheduled to execute complex tasks \cite{Caccavale17, Caccavale18}. When the DS is learned from demonstration, we are often interested in accurately reproducing the demonstrated trajectories. This is a key difference between DS-based and learning-based motion planners (like  \cite{pfeiffer2017from}), where the objective is to learn and generalize collision-free, goal-reaching paths given (imperfect) demonstrations. 
\begin{table*}[t!]
     \centering
     \caption{Comparison of different approaches for stable motion generation with dynamical systems.}
     \vspace{-0.2cm}
     \label{tab:comparison_appraoches}
     {\renewcommand\arraystretch{1.3} 
 	\begin{tabular}{ c||c|c|c|c|c|c }
 	 & Global stability & Multiple motions  & Any regression technique & Autonomous DS & Single-step learning & Training time\\
 	 \hline
     \hline 
     ESDS (Ours) & \cmark & \cmark & \cmark &  -  & \cmark & low \\
     SEDS \cite{SEDS} & \cmark & \cmark  & - &  \cmark  & - & high\\
     DMP \cite{DMP} & \cmark & - & - &  -  & \cmark & low\\
     CLF-DM \cite{Clf} & \cmark & \cmark  & \cmark &  \cmark  & - & medium\\
     $\tau$-SEDS \cite{tau-SEDS} & \cmark & \cmark  & \cmark & \cmark & - & high\\
     FDM \cite{Perrin16} & \cmark & - & - &  \cmark  & \cmark & low\\
     FSM-DS \cite{duan2017fast} & \cmark & \cmark & - & \cmark   & - & medium\\
     C-GMR \cite{Blocher17} & \cmark & \cmark & - &  -  &  \cmark & low\\
 \end{tabular}
 }
 \vspace{-0.5cm}
 \end{table*}

Dynamic movement primitives (DMPs) \cite{DMP} are certainly the most used approach for DS-based motion generation. A DMP is the superposition of a linear system and a non-linear term learned from a single demonstration. A time-dependent clock signal suppresses the non-linear term to enforce the convergence towards a given target. The DMP framework has been extended in several ways. For instance \cite{Pervez18, TP-DMP} introduce task-dependent parameters to customize the robot motion, while \cite{Saveriano19_Merging} investigates the possibility to merge multiple DMPs to plan more complex pose trajectories. A problem of DMPs is that the clock signal may introduce deviations from the demonstration. Although time rescaling techniques have been proposed \cite{DMP}, DMPs still miss generalization capabilities outside the demonstration area \cite{tau-SEDS}.

An alternative approach is to encode the motion into a state-dependent DS. This idea is exploited by the stable estimator of dynamical systems (SEDS) in \cite{SEDS}, where the parameters of a Gaussian mixture regression (GMR) are constrained to ensure global stability. However, SEDS exploits quadratic stability constraints that limit the reproduction accuracy. Researchers in the field have realized that, in some cases, accuracy and stability are conflicting objectives to achieve. This is known as the \textit{accuracy vs stability dilemma} \cite{tau-SEDS} and several approaches have been proposed to improve the accuracy while preserving the stability.

The work in \cite{Blocher17, duan2017fast} focus on a single regression technique and try to maximize the accuracy while keeping the stability. Similarly, \cite{Perrin16} proposes to transform a linear DS using a learned diffeomorphic mapping\footnote{A diffeomorphism is a bijective, continuous, and continuously differentiable mapping with a continuous and continuously differentiable inverse.} that accurately represents a single demonstration. These approaches have the advantage of handling the accuracy vs stability dilemma within a single-step learning procedure. However, it is clear that focusing on a single regression technique limits the applicability of the approach because it is not known a prior which regression technique works better for a given application. 

Alternative approaches apply a \textit{divide et impera} procedure and separate accuracy and stability into two different (two-steps) learning processes. The work in \cite{Clf} learns a flexible Lyapunov function which reduces the contradictions between stability constraints and training data. The parameters of this Lyapunov function are determined by solving a constrained optimization problem. The approach in \cite{tau-SEDS} assumes that there exists a diffeomorphed space where the training data are accurately represented by a quadratic Lyapunov function. The authors propose to learn a diffeomorphic mapping from training data, project the data into the diffeomorphed space, and then use an approach like SEDS to fit a stable DS. Although the approaches in \cite{Clf, tau-SEDS} result in accurate motions, the two-step learning requires to fit both a Lyapunov function (or a diffeomorphism) and a dynamical system from training data. Hence, these approaches introduce further tunable parameters and increase the overall training time. The main features of the presented approaches are listed in Tab.~\ref{tab:comparison_appraoches}.  

In this paper, we propose the Energy-based Stabilizer of Dynamical Systems (ESDS), a single-step learning framework that copes with the stability vs accuracy dilemma. ESDS exploits a particular form of DS consisting of a superposition of a conservative vector field (a linear dynamics) and a non-linear, in general dissipative one. The dissipative vector field is learned from demonstrations and retrieved at run-time using any regression technique. Therefore, ESDS is not targeted to a single regression technique like \cite{Perrin16, Blocher17, duan2017fast}. To stabilize the learned dynamics, ESDS proceeds with energy considerations in a way that resembles the energy tank-based controllers \cite{Franken11, Kronander16, Selvaggio19}, but with the objective to ensure the stability of the DS rather than its passivity. To this end, we augment a quadratic Lyapunov function with an additive term that plays the role of a virtual energy tank. With this formulation, the non-linear field is followed until there is energy in the tank and smoothly vanishes if the energy is depleted. Being the initial value of the energy of importance for an accurate reproduction, we propose an approach to determine the initial energy from training data. An experimental comparison shows that ESDS reaches similar or better accuracy than two-step approaches in a time that is comparable with single-step ones, which makes ESDS a promising approach for motion generation with stable DS.


\section{Proposed Approach}\label{sec:approach}
\subsection{Problem definition and background material}\label{subsec:problem}
We represent a robotic skill as a mapping between the robot position (in joint or Cartesian space) $\bfx(t) \in \mathbb{R}^n$ and its velocity $\dot{\bfx}(t) \in \mathbb{R}^n$. Therefore, the skill can be represented as a first-order dynamical system (DS) in the form
\begin{equation}
	\dot{\bfx} = \bfg(\bfx),
	\label{eq:ds_syst}
\end{equation}
where the time dependencies have been omitted to ease the notation. In general, the mapping occurs through the continuous and continuously differentiable non-linear function $\bfg : \mathbb{R}^{n} \rightarrow \mathbb{R}^{n}$. A solution $\bfPhi(\bfx_0,t) \in \mathbb{R}^{n}$ of \eqref{eq:ds_syst} represents a trajectory. Different initial conditions $\bfx_0$ generate different trajectories. An equilibrium point is a point where the velocity vanishes, i.e. $\hat{\bfx} : \bfg(\hat{\bfx}) = \mathbf{0} \in \mathbb{R}^{n}$. $\hat{\bfx} \in S \subset \mathbb{R}^{n}$ is locally asymptotically stable (LAS) if $\lim_{t\rightarrow+\infty} \bfPhi(\bfx_0,t) =\hat{\bfx}, \forall \bfx_0 \in S$. If $S = \mathbb{R}^{n}$, $\hat{\bfx}$ is globally asymptotically stable (GAS) and it is the only equilibrium of the DS.

A sufficient condition for $\hat{\bfx}$ to be GAS is that there exists a scalar, continuously differentiable function of the state variables $\mathcal{V}(\bfx) \in \mathbb{R}$ satisfying \cite{Slotine91}:
\begin{subequations}
	\begin{align}
		&\mathcal{V}(\bfx) \geq 0, ~\forall \bfx \in \mathbb{R}^{n} ~\text{and}~ \mathcal{V}(\bfx) = 0 \Longleftrightarrow \bfx = \hat{\bfx} \label{eq:lyap_stab_cond1}\\
		&\dot{\mathcal{V}}(\bfx) \leq 0, ~\forall \bfx \in \mathbb{R}^{n} ~\text{and}~  \dot{\mathcal{V}}(\bfx) = 0 \Longleftrightarrow \bfx = \hat{\bfx} \label{eq:lyap_stab_cond2}\\
		&\mathcal{V}(\bfx) \rightarrow \infty ~\text{as}~ \Vert \bfx \Vert  \rightarrow \infty ~\text{(radially unbounded)} \label{eq:lyap_stab_cond3}
	\end{align}	
\end{subequations} 
Note that, if condition (\ref{eq:lyap_stab_cond3}) is not satisfied, the equilibrium point is LAS. A function that satisfies conditions \eqref{eq:lyap_stab_cond1}--\eqref{eq:lyap_stab_cond2} is called a Lyapunov function.

Our goal is to represent a set of training data (position/velocity pairs) as a first-order DS that satisfies the following requirements:
\begin{enumerate}
	\item The point $\hat{\bfx}$ is a GAS---and therefore unique---equilibrium of the DS.
	\item The trajectories of the DS ``accurately'' represent the training data.
\end{enumerate}
An approach that satisfies 1) and 2) is presented as follows. In the rest of the paper, we consider $\hat{\bfx}=\zero$ without loss of generality. A different goal $\hat{\bfx}'\neq \zero$ can be reached by applying the constant state translation $\bfy = \bfx - \hat{\bfx}'$.

\subsection{Dynamical system definition}\label{subsec:learmed_ds}
In this work, we use a DS different from \eqref{eq:ds_syst}. More specifically, we consider the non-linear function $\bfg(\bfx)$ as the superposition of two vector fields: \textit{i)} a stable, linear vector field $-\bfx$ and \textit{ii)} a non-linear, in general non-conservative, vector field $\bff(\bfx)$. The DS used in this work is
\begin{equation}
\dot{\bfx} = \bfg(\bfx) = -\bfx + \kappa(\Vert \bfx \Vert) \bff(\bfx) 
\label{eq:paramet_ds}
\end{equation}
where $\kappa(\cdot)$ is a smooth and positive semidefinite function that vanishes at the equilibrium ($\kappa(\Vert\hat{\bfx}\Vert)=0$). The function $\kappa(\cdot)$ is used to ensure that the DS in \eqref{eq:paramet_ds} has an equilibrium point at $\hat{\bfx}=\zero$. We use $\kappa(\Vert\bfx\Vert) = 1 - e^{-a\Vert \bfx \Vert^2}$ in this work. 

At this point, we have no guarantee that $\hat{\bfx}$ is GAS, since the DS can fall in a spurious equilibrium, follow a periodic orbit, or even diverge. The stability of \eqref{eq:paramet_ds} is analyzed in Sec.~\ref{subsec:stability}, while the rest of this section explains how to encode a set of demonstrations in a DS defined as in \eqref{eq:paramet_ds}.    
Consider that a set of demonstrations is given as $\{ \bfx_{d}^{t}, ~\dot{\bfx}_{d}^{t}\}_{t=1,d=1}^{T,D}$, where $\bfx_{d}^{t} \in \mathbb{R}^{n}$ is the desired robot position (in joint or Cartesian space) at time $t$ and demonstration $d$, $\dot{\bfx}_{d}^{t}\in \mathbb{R}^n$ is the desired robot velocity, $T$ is the number of samples in each demonstration, and $D$ the number of demonstrations. The demonstrations are firstly converted into input/output pairs by rewriting the DS in \eqref{eq:paramet_ds} as
\begin{equation}
\label{eq:train_ds}
\kappa^{-1}(\Vert\bfx\Vert)\left(\dot{\bfx} + \bfx\right) = \bff(\bfx), 
\end{equation}
where
\begin{equation}
\kappa^{-1}(\Vert\bfx\Vert) = \begin{cases}
\frac{1}{\kappa(\Vert\bfx\Vert)} \quad  \Vert\bfx\Vert \neq 0 \\
			1 ~~\qquad \text{otherwise}
		   \end{cases}.
\label{eq:kappa_inverse}
\end{equation}
Equation \eqref{eq:train_ds} shows that the unknown term $\bff(\bfx)$ is a non-linear mapping between $\bfx$ and $\kappa^{-1}(\Vert\bfx\Vert)\left(\dot{\bfx} + \bfx\right)$. Therefore, we use the demonstrated states $\bfX = \{\bfx_{d}^{t}\}_{t=1,d=1}^{T,D}$ as input and $\bfY = \left\lbrace \kappa^{-1}(\Vert\bfx_{d}^{t}\Vert)\left(\dot{\bfx}_{d}^{t} + \bfx_{d}^{t}\right)\right\rbrace_{t=1,d=1}^{T,D}$ as observations of $\bff(\bfx)$ in a supervised learning process. In other words, we add a non-linear, state-dependent displacement $\kappa( \Vert\bfx \Vert) \bff(\bfx)$ to the linear and convergent dynamics $-\bfx$ to closely reproduce the demonstrated velocities. 
Given the input/output pairs $\bfX$/$\bfY$, any regression technique can be used to learn $\bff(\bfx)$ and retrieve a smooth velocity for each state. Note that the definition of $\kappa^{-1}$ in \eqref{eq:kappa_inverse} does not generate discontinuities because $\dot{\bfx}_{d} + \bfx_{d}\rightarrow \zero$ for $\bfx_{d}\rightarrow \zero$.

\subsection{Energy tank-based Lyapunov function}\label{subsec:energy_lyapunov}
As mentioned in Sec.~\ref{subsec:learmed_ds}, the stability of \eqref{eq:paramet_ds} is not guaranteed yet. To this end, we aim at exploiting the Lyapunov stability results summarized in Sec.~\ref{subsec:problem}. The easiest way to define a function that satisfies conditions \eqref{eq:lyap_stab_cond1} and \eqref{eq:lyap_stab_cond3}, i.e. a candidate Lyapunov function, is to consider the quadratic potential $\mathcal{V}(\bfx) = 0.5\Vert\bfx\Vert^2$.
In order to fulfill also condition \eqref{eq:lyap_stab_cond2}, one has to show that the time derivative $\dot{\mathcal{V}}(\bfx) < 0 ~\forall \bfx \neq \zero$, i.e. one has to analyze the sign of
\begin{equation}
\dot{\mathcal{V}}(\bfx) =  \frac{\partial\mathcal{V}}{\partial\bfx}\dot{\bfx} = \nabla \mathcal{V} \dot{\bfx} = -\bfx\tr\bfx + \kappa(\Vert\bfx\Vert)\bfx\tr\bff(\bfx),
\label{eq:stab_quadratic_lyap}
\end{equation}
where the definition of $\dot{\bfx}$ from \eqref{eq:paramet_ds} has been used. By inspecting \eqref{eq:stab_quadratic_lyap}, it is straightforward to verify that $\dot{\mathcal{V}}(\hat{\bfx})=0$, and that the term $-\bfx\tr\bfx$ is always negative. The problem is that the sign of $\kappa(\Vert\bfx\Vert)\bfx\tr\bff(\bfx)$ cannot be determined in advance even if we know that $\kappa(\Vert\bfx\Vert) \geq 0$. Therefore, the stability condition \eqref{eq:lyap_stab_cond2} is not automatically verified. Several approaches have been proposed to ensure that $\bfx\tr\bff(\bfx) \leq 0$ (see, for example, \cite{SEDS, tau-SEDS}). However, it is known that the stability conditions imposed by a quadratic potential may prevent accurate motion learning \cite{Blocher17, Clf}. 

Instead of imposing quadratic stability constraints on the DS, we propose an alternative solution inspired by energy tank-based controllers \cite{Franken11, Kronander16,Selvaggio19}. The idea is to associate to the DS a virtual energy tank and reuse the energy dissipated when $\kappa(\Vert\bfx\Vert)\bfx\tr\bff(\bfx)\leq 0$. Stable motions generated by $\kappa(\Vert\bfx\Vert)\bfx\tr\bff(\bfx)\leq 0$ increase the level of the tank while possibly unstable motions ($\kappa(\Vert\bfx\Vert)\bfx\tr\bff(\bfx)> 0$) reduce the level of the tank. In this formulation, possibly unstable motions are executed unless the storage is depleted, preserving the overall stability of the system. More formally, we consider the augmented Lyapunov candidate \begin{equation}
\mathcal{V}_s(\bfx) =  \mathcal{V}(\bfx) + s = \frac{1}{2}\Vert\bfx\Vert^2 + s,
\label{eq:energy_tank_Lyapunov}
\end{equation}
where the function $s$ plays the role of an energy tank and keeps track of the energy dissipated by the DS in previous instants. Recall that a candidate Lyapunov function satisfies conditions \eqref{eq:lyap_stab_cond1} and \eqref{eq:lyap_stab_cond3}, i.e. it is positive definite and radially unbounded. Assuming that $s$ is positive semi-definite and zero at least for $\hat{\bfx} = \zero$, condition \eqref{eq:lyap_stab_cond1} is satisfied. Notice that, even if $s=0$ for some $\bfx' \neq \zero$, $\mathcal{V}_s(\bfx') > 0$. Condition \eqref{eq:lyap_stab_cond3} can be fulfilled by requiring that $s\rightarrow+\infty$ for $\Vert \bfx \Vert \rightarrow+\infty$ or, as in this work, by assuming that $s$ is bounded.

In principle, it is possible to assign to $s$ an arbitrary dynamics that results in a positive semi-definite and bounded $s$ that vanishes at $\hat{\bfx} = \zero$. Let us first define the dynamics
\begin{equation}
\dot{s}' = \alpha(s)\Vert \bfx\Vert^2-\beta(z,s)z.
\label{eq:storage_classic}
\end{equation}
We can now define the dynamics of $s$ as
\begin{equation}
\dot{s} = \begin{cases}
\dot{\kappa}(\Vert \bfx \Vert)\overline{s} \quad  s \geq \kappa( \Vert\bfx\Vert )\overline{s} ~\text{and}~ \dot{s}' > \dot{\kappa}(\Vert \bfx \Vert)\overline{s} < 0 \\
			\dot{s}' ~~ \qquad \text{otherwise}
		   \end{cases},
\label{eq:storage_dynamics}
\end{equation}
where $z = \kappa(\Vert \bfx \Vert)\bfx\tr\bff(\bfx)$, $\kappa(\Vert \bfx \Vert)$ is defined in Sec.~\ref{subsec:learmed_ds}, and $\overline{s}$ is the maximum value of $s$. Indeed, the storage is initialized as $s_0 = \kappa(\Vert \bfx \Vert)\overline{s}$, with $0\leq\kappa(\cdot)\leq 1$. Note that the first condition in~\eqref{eq:storage_dynamics} guarantees that both $s$ and $\mathcal{V}_s(\bfx)$ in \eqref{eq:energy_tank_Lyapunov} vanish at the equilibrium point. The gain $\alpha(s)$ satisfies
\begin{equation}
	\begin{cases}
		0 \leq \alpha(s) < 1 \quad s < \kappa(\Vert \bfx \Vert)\overline{s} \\ 
		\alpha(s) = 0 ~~~\qquad \text{otherwise}
	\end{cases},
	\label{eq:conditions_alpha}
\end{equation}
where $\alpha(s)$ is strictly smaller than $1$ to ensure global stability (see Sec.~\ref{subsec:stability}). The scalar gain $\beta(z,s)$ satisfies
\begin{equation}
	\begin{cases}
    	\beta(z,s) = 0 ~~\qquad s \geq \kappa(\Vert \bfx \Vert)\overline{s} ~\text{and}~ z < 0 \\ 
		\beta(z,s) = 0 ~~\qquad s \leq \underline{s}~\text{and}~ z \geq 0 \\ 
		0 \leq \beta(z,s) \leq 1 \quad \text{otherwise}
	\end{cases}.
	\label{eq:conditions_beta}
\end{equation}
The parameter $\underline{s} \geq 0$ is the minimum value of $s$ and we use $\underline{s}=0$ in this work. The choice of $\alpha(s)$ and $\beta(z,s)$ is made to ensure that $s_0 = s(0) \geq \underline{s} = 0 \Rightarrow s(t) \geq \underline{s} = 0, ~\forall t>0$. In this way, the function $\mathcal{V}_s(\bfx)$ is positive definite. Moreover, the dynamics of $s$ in \eqref{eq:storage_dynamics} guarantees that $0 \leq s \leq \kappa(\Vert \bfx \Vert)\overline{s}$ if $s_0 \leq \kappa(\Vert \bfx \Vert)\overline{s}$. Recalling that $\kappa( 0 ) = 0$, this allows us to conclude that $\mathcal{V}_s(\hat{\bfx})=0$. The function $\mathcal{V}_s(\bfx)$ is also radially unbounded since $s$ is bounded and $\Vert \bfx \Vert$ is unbounded. 


%
%
%
%


The non-linear term $\kappa(\Vert \bfx \Vert)\bff(\bfx)$ of the DS \eqref{eq:paramet_ds} can be followed until there is energy in the storage. Indeed, one can extract energy from the storage ($z>0$) if and only if $s>0$. On the contrary, when  the  storage  is  depleted ($s=0$), the non-linear term may compromise the stability and it has to be suppressed. To this end, we modify the DS in \eqref{eq:paramet_ds} as
\begin{equation}
\dot{\bfx} = -\bfx + \gamma(z,s)\kappa(\Vert \bfx \Vert)\bff(\bfx), 
\label{eq:paramet_ds_gamma}
\end{equation}
where the scalar gain $\gamma(z,s)$ satisfies
\begin{equation}
	\begin{cases}
		\gamma(z,s) = \beta(z,s) \quad z \geq 0 \\ 
		\gamma(z,s) \geq \beta(z,s) \quad \text{otherwise}
	\end{cases}.
	\label{eq:conditions_gamma}
\end{equation}
By inspecting \eqref{eq:paramet_ds_gamma} and \eqref{eq:conditions_gamma}, it is clear that $\gamma(z,s)$ suppresses the term $\kappa(\Vert\bfx\Vert)\bff(\bfx)$ when the storage is depleted. In this case, the DS is driven towards the equilibrium by the linear---and therefore stable---term $-\bfx$. The functions $\alpha(s)$, $\beta(z,s)$, and $\gamma(z,s)$ are defined in Tab.~\ref{tab:smooth_functions}. These piecewise-defined functions are chosen to guarantee smooth switching between different conditions. It is worth noticing that the formulation in~\eqref{eq:conditions_alpha}--\eqref{eq:conditions_gamma} resembles that in \cite{Kronander16}. In \cite{Kronander16}, torque commands for the robot are generated by feeding the robot state in a given DS. The torque command has a non-passive term that is suppressed by a function similar to $\gamma(\cdot)$. Here, we apply a similar approach with the goal of learning an accurate and converging open-loop representation of a demonstrated motion. As a final remark, $s$ introduces a time dependency in the dynamical system \eqref{eq:paramet_ds_gamma}. This DS formulation \eqref{eq:paramet_ds} is then similar to DMPs \cite{DMP}, in the sense that DMPs also use a time-dependent (clock) signal to retrieve asymptotic stability. However, in the DMP formulation, the non-linear term only depends on time and is learned from a single demonstration, while in ESDS the term $\bff(\bfx)$ is state dependent and can be learned from multiple demonstrations. Moreover, the clock signal in the DMP exponentially vanishes with time, while the storage $s$ can be potentially charged during the motion, helping to accurately reproduce the demonstrations.

\begin{table}[t]
    \centering
    \caption{Definition of the functions $\alpha(s)$, $\beta(z,s)$, $\gamma(z,s)$, and $\kappa(\Vert\bfx\Vert)$.} 
    \label{tab:smooth_functions}
    \resizebox{\columnwidth}{!}{%
    {\renewcommand\arraystretch{1.3} 
	\begin{tabular}{ |c| }
	\hline
	$
	h_1(x,\underline{x},\overline{x}) = \begin{cases}
		1 \quad\qquad\qquad\qquad \qquad \qquad \qquad x \geq \overline{x} \\ 
		0 \quad\qquad\qquad\qquad \qquad \qquad \qquad x \leq \underline{x} \\ 
		0.5(1+\sin\left(\pi\left(\frac{x-\underline{x}}{\overline{x}-\underline{x}} - 0.5\right)\right) ~\text{otherwise}
	\end{cases}
	$\\
	\hline
	\hline
	$h_2(x,\underline{x},\overline{x}) = 1 - h_1(x,\underline{x},\overline{x})$\\
	\hline
	\hline
	$\alpha(s) = \min(0.99,h_1(s,0,0.1\kappa(\Vert \bfx \Vert)\overline{s})\cdot h_2(s,0.9\kappa(\Vert \bfx \Vert)\overline{s},\kappa(\Vert \bfx \Vert)\overline{s}))$\\
	\hline
	\hline
	$\beta(z,s) = 1 - h_1(z,-0.01,0)\cdot h_2(s,0,0.1\kappa(\Vert \bfx \Vert)\overline{s})$\\
	$ \qquad \qquad \qquad - h_1(s,0.9\kappa(\Vert \bfx\Vert )\overline{s},\kappa(\Vert \bfx \Vert)\overline{s})\cdot h_2(z,0,0.01)$\\
	\hline
	\hline
	$\gamma(z,s) = 1 - h_1(z,0,0.01)\cdot h_2(s,0,0.1\kappa(\Vert \bfx \Vert)\overline{s})$\\
	\hline
	$\kappa(\Vert\bfx\Vert) = 1 - e^{-0.1\Vert\bfx\Vert^2}$\\
	\hline
\end{tabular}
}}
\end{table}	

 \subsection{Stability analysis}\label{subsec:stability}
The candidate Lyapunov function in~\eqref{eq:energy_tank_Lyapunov} and the DS~\eqref{eq:paramet_ds_gamma} allow us to prove the following stability theorem.
\begin{theorem}
\label{th:Lyap_tank_stabilty}
Assume that the dynamics of the storage function $s$ is defined as in \eqref{eq:storage_dynamics} with $\dot{s}'$ defined as in~\eqref{eq:storage_classic}. Assume also that $\alpha(s)$, $\beta(z,s)$, and $\gamma(z,s)$ satisfy conditions \eqref{eq:conditions_alpha},~\eqref{eq:conditions_beta}, and \eqref{eq:conditions_gamma} respectively, and that $0 \leq s_0 \leq \kappa(\Vert \bfx\Vert )\overline{s}$. Under these assumptions, the dynamical system defined as in \eqref{eq:paramet_ds_gamma} globally asymptotically converges to $\hat{\bfx} = \zero$.
\end{theorem}
\begin{proof}
The point $\hat{\bfx} = \zero$ is an equilibrium of \eqref{eq:paramet_ds_gamma} because $\kappa(\Vert\hat{\bfx}\Vert)=0$. Therefore, we have to show that 
$\mathcal{V}_s$ in \eqref{eq:energy_tank_Lyapunov} is a Lyapunov function for \eqref{eq:paramet_ds_gamma}.

First, we show that $\mathcal{V}_s$  is a proper the candidate Lyapunov function, i.e. it satisfies \eqref{eq:lyap_stab_cond1} and \eqref{eq:lyap_stab_cond3}. $\mathcal{V}_s$ is radially unbounded (condition \eqref{eq:lyap_stab_cond3}) because $s$ is bounded. Moreover, since the dynamics of the storage function $s$ is defined as in \eqref{eq:storage_dynamics}, it holds that $s(\hat{\bfx}=\zero)=0$. Since $\alpha(s)$ and $\beta(z,s)$ satisfy conditions \eqref{eq:conditions_alpha} and \eqref{eq:conditions_beta} respectively, and being $s_0 \geq 0$, it also holds that $s_t \geq 0,~\forall t$. Consequently, $\mathcal{V}_s > 0, ~\forall \bfx\neq\zero$ and $\mathcal{V}_s = 0$ only at $\bfx = \zero$ (condition \eqref{eq:lyap_stab_cond1}). 

Second, we prove that condition \eqref{eq:lyap_stab_cond2} holds. By taking the time derivative of $\mathcal{V}_s$ we have that $\dot{\mathcal{V}}_s = \bfx\tr\dot{\bfx} + \dot{s}$. The function $\dot{\mathcal{V}}_s$ vanishes at the equilibrium because $\hat{\bfx}\tr\dot{\bfx} = \dot{s}(\hat{\bfx})=0$. Therefore, we have to prove that $\dot{\mathcal{V}}_s < 0, \forall \bfx \neq \zero$. Considering the definition of $\dot{\bfx}$ in \eqref{eq:paramet_ds_gamma} and $\dot{s}$ in \eqref{eq:storage_dynamics}, we have to consider two cases.

\textit{Case I: } In this case $\dot{s} = \dot{s}' = \alpha(s)\Vert\bfx \Vert^2 - \beta(z,s)z$, hence 
\begin{equation}
\dot{\mathcal{V}}_s = -(1-\alpha(s))\Vert \bfx \Vert^2 + (\gamma(z,s)-\beta(z,s))z.
\label{eq:V_tank_dot_cont}
\end{equation}
Now, being $0\leq \alpha(s) < 1$ from \eqref{eq:conditions_alpha} and $\Vert \bfx \Vert^2 \geq 0$, the term $-(1-\alpha(s))\Vert \bfx \Vert^2$ in \eqref{eq:V_tank_dot_cont} is negative and vanishes only at $\bfx=\zero$. The second term $v = (\gamma(z,s)-\beta(z,s))z$ also vanishes at $\bfx=\zero$ because $z=\kappa(\Vert\bfx\Vert)\bfx\tr\bff(\bfx)$. Moreover, being $\gamma(z,s)-\beta(z,s) = 0$ for $z \geq 0$ and  $\gamma(z,s)-\beta(z,s) \geq 0$ for $z < 0$, it holds that $v \leq 0, \forall \bfx$. This allows as to conclude that $\dot{\mathcal{V}}_s < 0, \forall \bfx \neq \zero$ and $\dot{\mathcal{V}}_s = 0$ for $\bfx = \zero$ (condition \eqref{eq:lyap_stab_cond2}).

\textit{Case II: } The dynamics of $s$ is $\dot{s}=\dot{\kappa}(\Vert\bfz\Vert)\overline{s}$, hence
\begin{equation}
\dot{\mathcal{V}}_s = -\Vert \bfx \Vert^2 + \gamma(z,s)z + \dot{\kappa}(\Vert\bfx\Vert)\overline{s}.
\label{eq:V_tank_dot_cont_2}
\end{equation}
The terms $-\Vert \bfx \Vert^2$ and $\dot{\kappa}(\Vert\bfx\Vert)\overline{s}$ in \eqref{eq:V_tank_dot_cont_2} are negative, while the sign of $\gamma(z,s)z$ depends on the sign of $z$. For $z\leq 0$, it holds that $\dot{\mathcal{V}}_s<0, \forall \bfx \neq \zero$. For $z>0$, it holds from \eqref{eq:conditions_gamma} that $\gamma(z,s)=\beta(z,s)$. Being $s\geq\kappa(\Vert\bfx\Vert)\overline{s}$, it holds from \eqref{eq:conditions_alpha} that $\alpha(s) = 0$. From \eqref{eq:storage_classic}, $\alpha(s) = 0$ implies that $\beta(z,s)z = -\dot{s} $. Given these equalities and inspecting \eqref{eq:storage_dynamics}, it is straightforward to verify that $\gamma(z,s)z = \beta(z,s)z = -\dot{s}' < \dot{\kappa}(\Vert\bfx\Vert)\overline{s}$, that implies $\dot{\mathcal{V}}_s<0, \forall \bfx \neq \zero$ and $\dot{\mathcal{V}}_s = 0$ for $\bfx = \zero$.

Being condition \eqref{eq:lyap_stab_cond2} satisfied in all cases, we conclude that the DS in \eqref{eq:paramet_ds_gamma} has a GAS equilibrium at $\hat{\bfx} = \zero$. 
\end{proof}
 
\begin{remark}
Theorem~\ref{th:Lyap_tank_stabilty} still holds if the DS converges towards a different equilibrium $\hat{\bfx}'\neq\zero$. One has to simply apply the constant state translation $\bfy = \bfx - \hat{\bfx}'$ and define the DS as $\dot{\bfy} = -\bfy + \gamma(z,s)\kappa(\Vert \bfy \Vert)\bff(\bfy)$.  
\end{remark}

\subsection{Storage function initialization}\label{subsec:storage_init}
The initial value of the storage function affects the trajectory retrieved from the DS \eqref{eq:paramet_ds_gamma} and, as a consequence, the overall accuracy in reproducing the demonstrations. As qualitatively shown in Fig.~\ref{fig:storage_effects}, small values of $s_0 = \kappa(\Vert \bfx \Vert)\overline{s}$ cause the storage to be depleted too quickly introducing deviations from the demonstrated data (Fig.~\ref{fig:storage_effects}(a)~and~\ref{fig:storage_effects}(b)). For this reason, we propose an approach to estimate the value of $\overline{s}$ from training data. 
\begin{figure}[t]
	\centering
	\subfigure[$\overline{s} = 100\,$J]{\includegraphics[width=0.32\columnwidth]{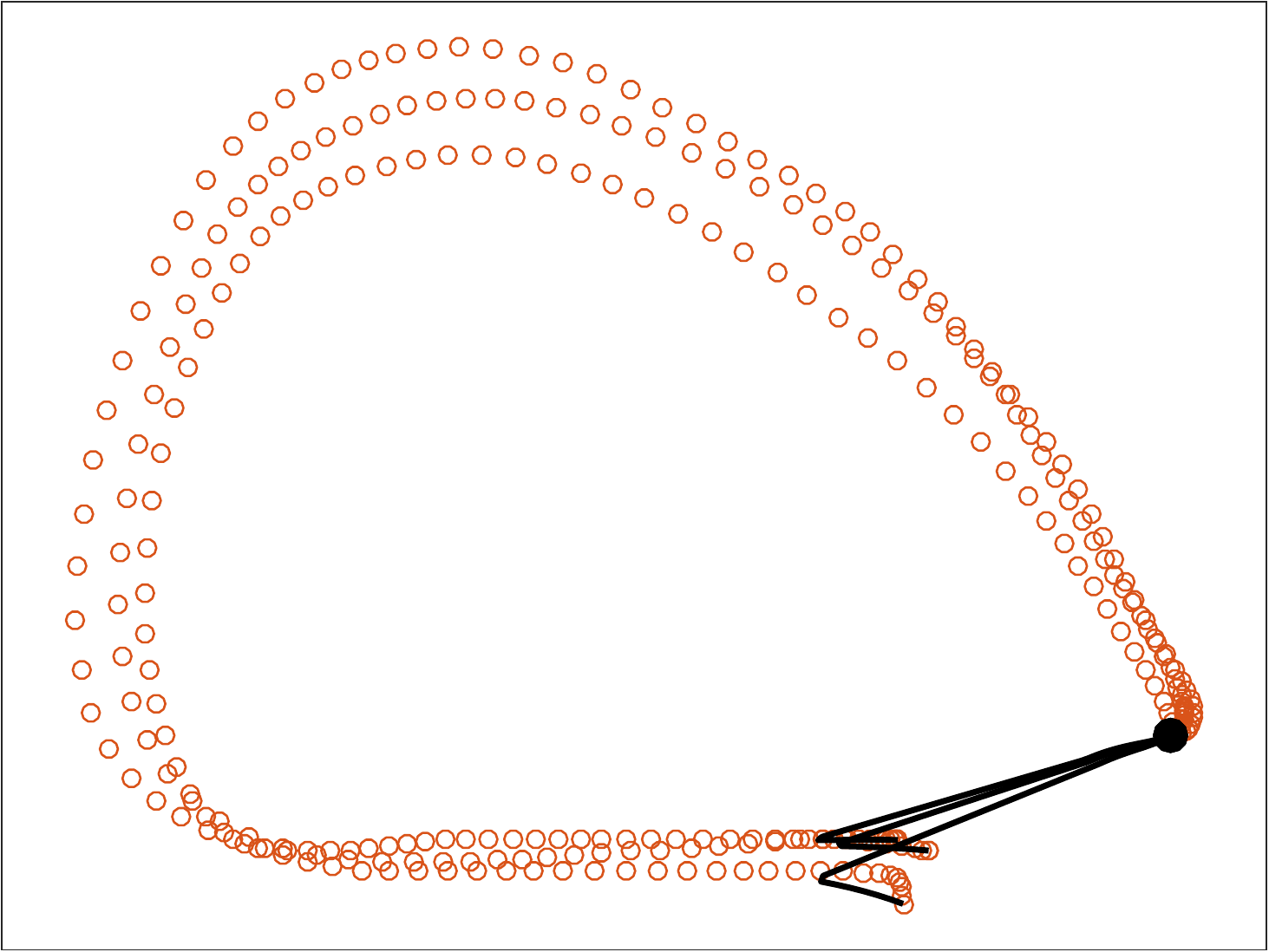}\label{fig:storage_effects_100}}
	\subfigure[$\overline{s} = 1000\,$J]{\includegraphics[width=0.32\columnwidth]{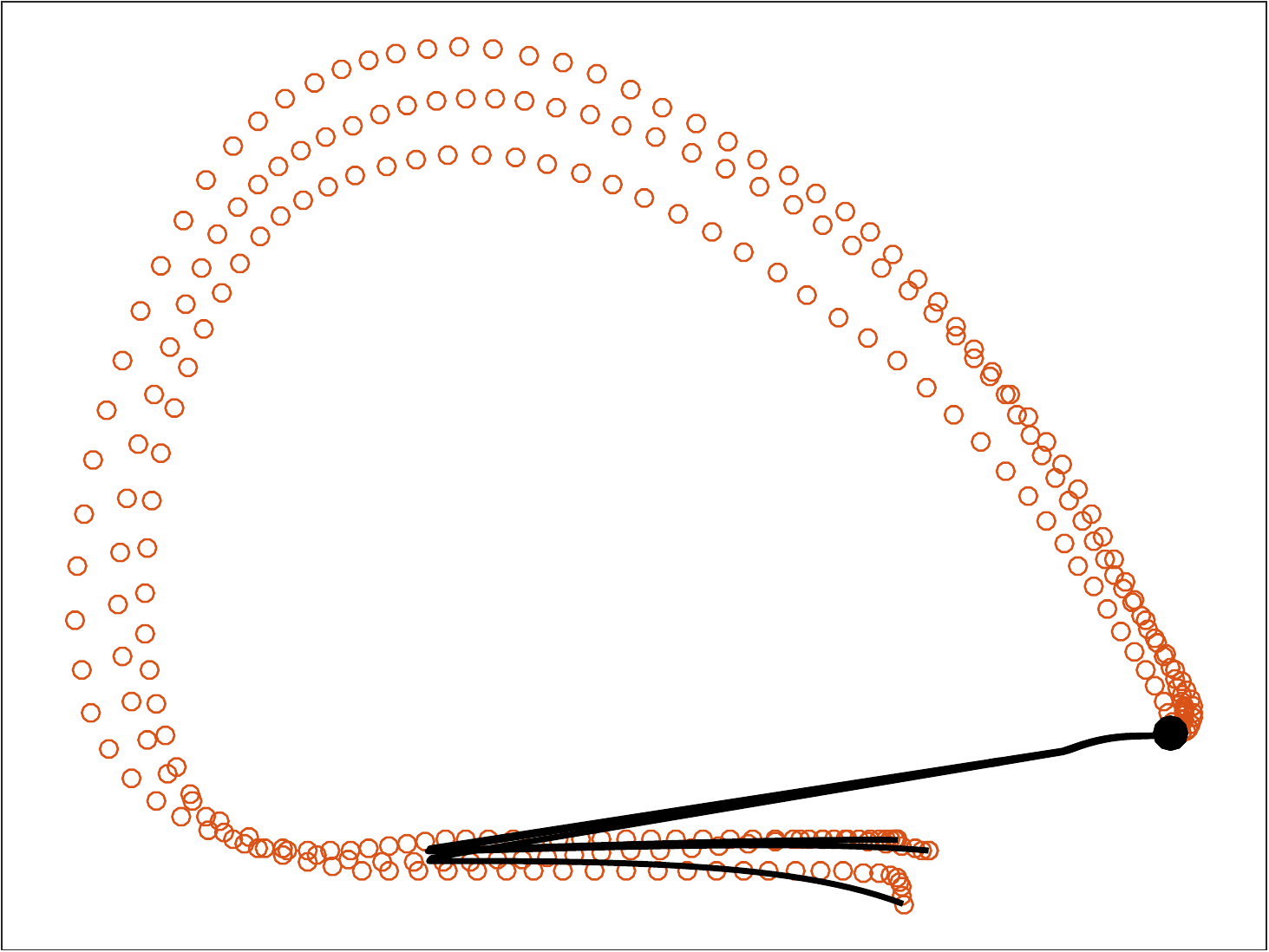}\label{fig:storage_effects_1000}}
	\subfigure[$\overline{s} = 10000\,$J]{\includegraphics[width=0.32\columnwidth]{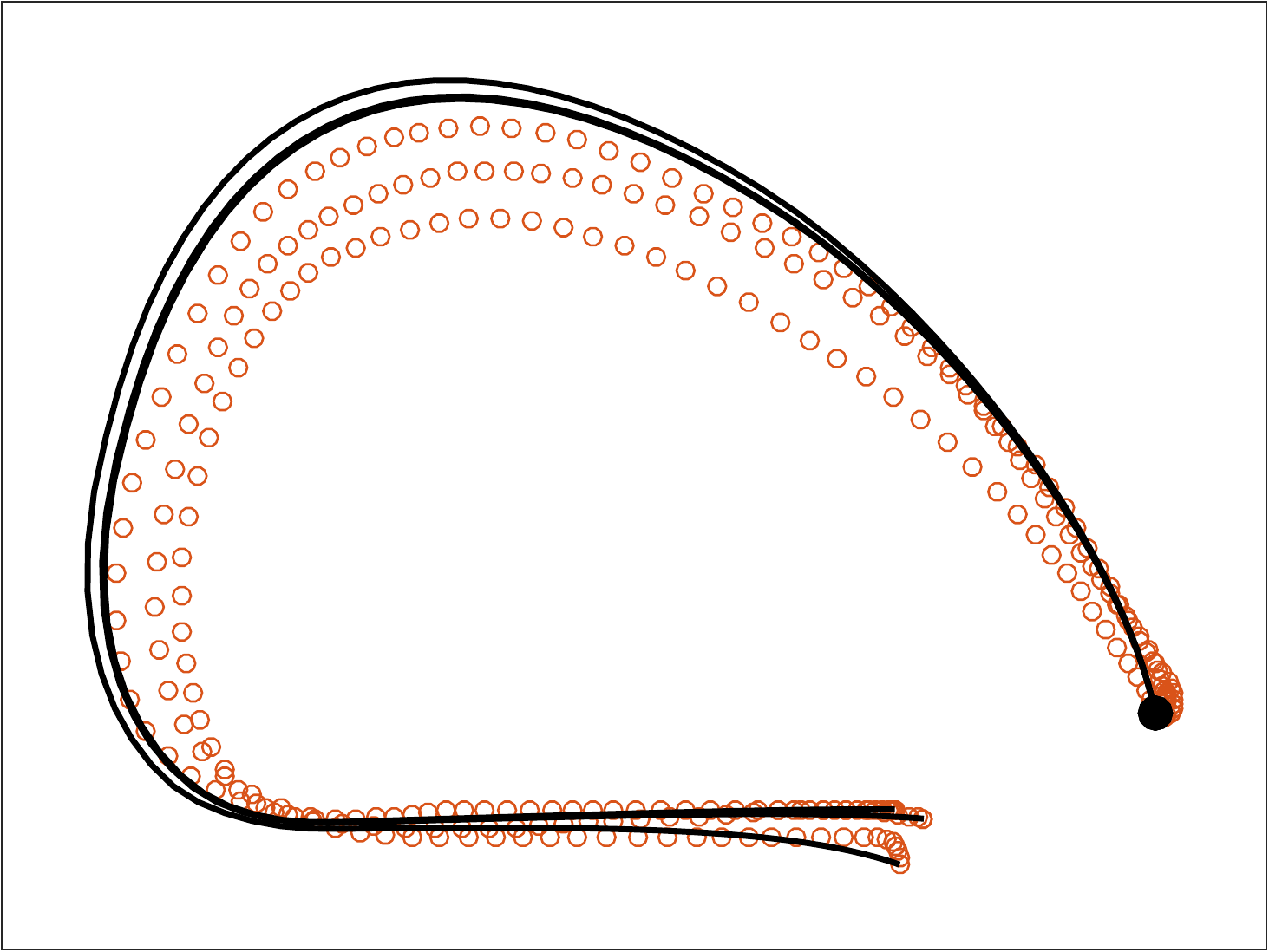}\label{fig:storage_effects_10000}}
    	\caption{Qualitative effects on the generated motions of to different choices of $s_0$. Brown circles represent the demonstrated positions, black solid lines the retrieved trajectories.}
    	\label{fig:storage_effects}
\end{figure} 
To this end, we consider
\begin{equation}
	\dot{s} = \begin{cases}
			   	-\bfx\tr\bff(\bfx) \quad \bfx\tr\bff(\bfx) > 0 \\ 
			    0 ~~\qquad \qquad \text{otherwise}
	\end{cases},
	\label{eq:storage_dynamics_worst}
\end{equation}
that is the dynamics in \eqref{eq:storage_classic} with $\alpha(s)=0$, $\kappa(\Vert\bfx\Vert)=1$, and $\beta(z,s)=1$ for $z > 0$ and $\beta(z,s)=0$ for $z \leq 0$. Notice that, being $\alpha(s)=0$, we do not consider the possible charge introduced by $\Vert \bfx \Vert^2$. Moreover, being $\beta(z,s)=0$ for $z \leq 0$, we do not consider the possible charge introduced by $z \leq 0$. Therefore, the storage dynamics in \eqref{eq:storage_dynamics_worst} only considers the depleted energy. Now, we compute $\dot{s}_d^t = -z_d^t = -\kappa(\Vert\bfx_d^t\Vert)(\bfx_d^t)\tr\bff(\bfx_d^t)$ for all the training data ($t=1,\ldots,T$ and $d=1,\ldots,D$), and select
\begin{equation}
\overline{s} = \max_{d = 1,\ldots,D} \overline{s}_{d} = \max_{d = 1,\ldots,D} \sum_{t=1}^T \dot{s}_d^t \delta t,
\label{eq:initial_storage}
\end{equation}
where $\delta t$ is the sampling time. Each $\overline{s}_{d}$ represents an upper bound on the energy depleted by the system \eqref{eq:paramet_ds_gamma} to reproduce the $d$-th demonstration and, by selecting $\overline{s} = \max_{d = 1,\ldots,D} \overline{s}_{d}$ and $s_0 = \kappa(\Vert \bfx \Vert)\overline{s}$, there will be enough energy to accurately reproduce the demonstrations without affecting the stability. Note that the $\overline{s}$ in~\eqref{eq:initial_storage} does not guarantee that the storage is not depleted for any initial condition but only sufficiently close to the demonstrations. However, even if the storage is depleted before reaching the goal, the DS~\eqref{eq:paramet_ds_gamma} is still stable (Theorem~\ref{th:Lyap_tank_stabilty}) and generated motion converges to the goal.
 



\section{Results and Comparisons}\label{sec:lasa_test}
\begin{figure*}[t]
	\centering
	\includegraphics[width=0.07\textwidth]{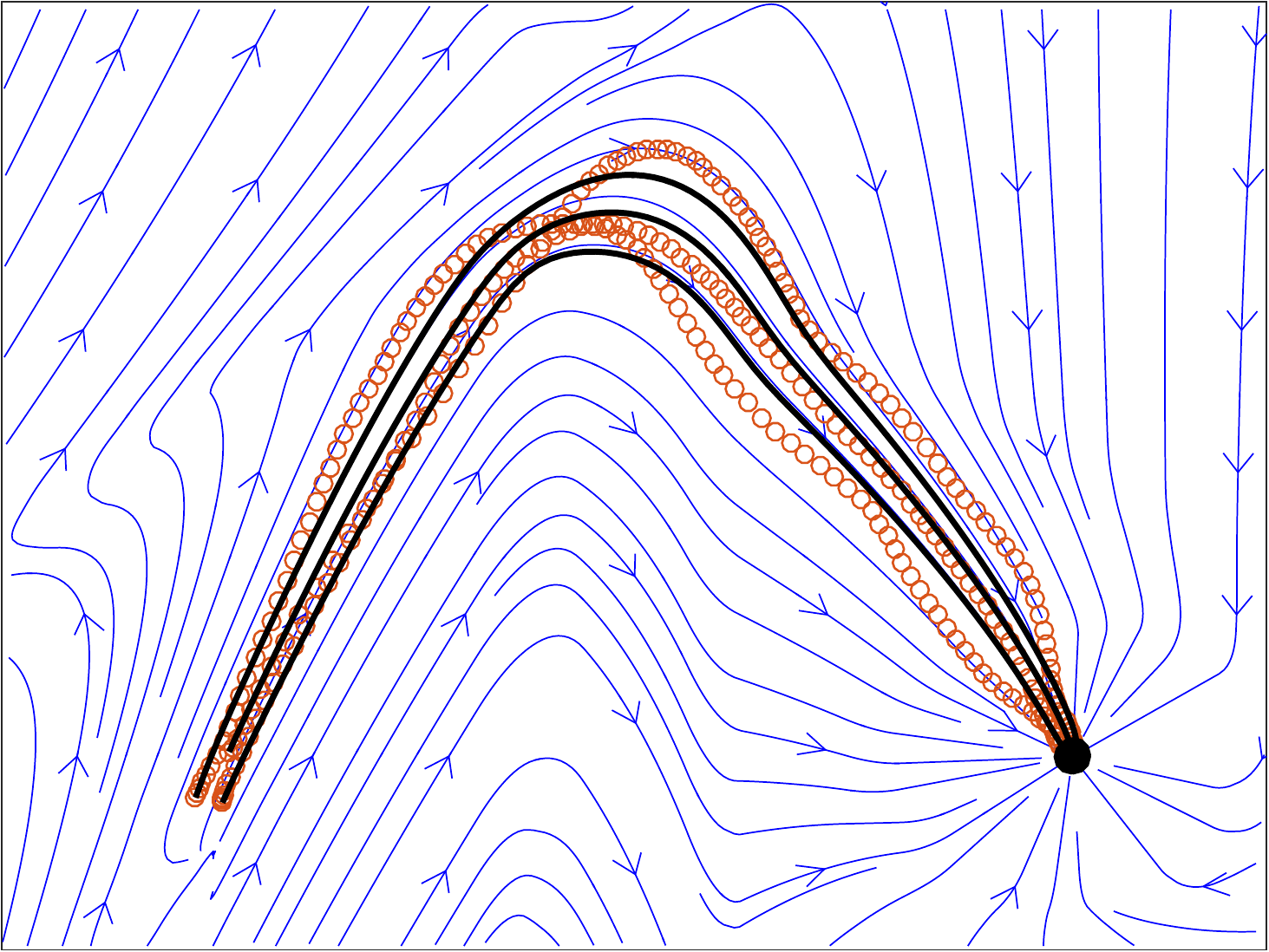}
	\includegraphics[width=0.07\textwidth]{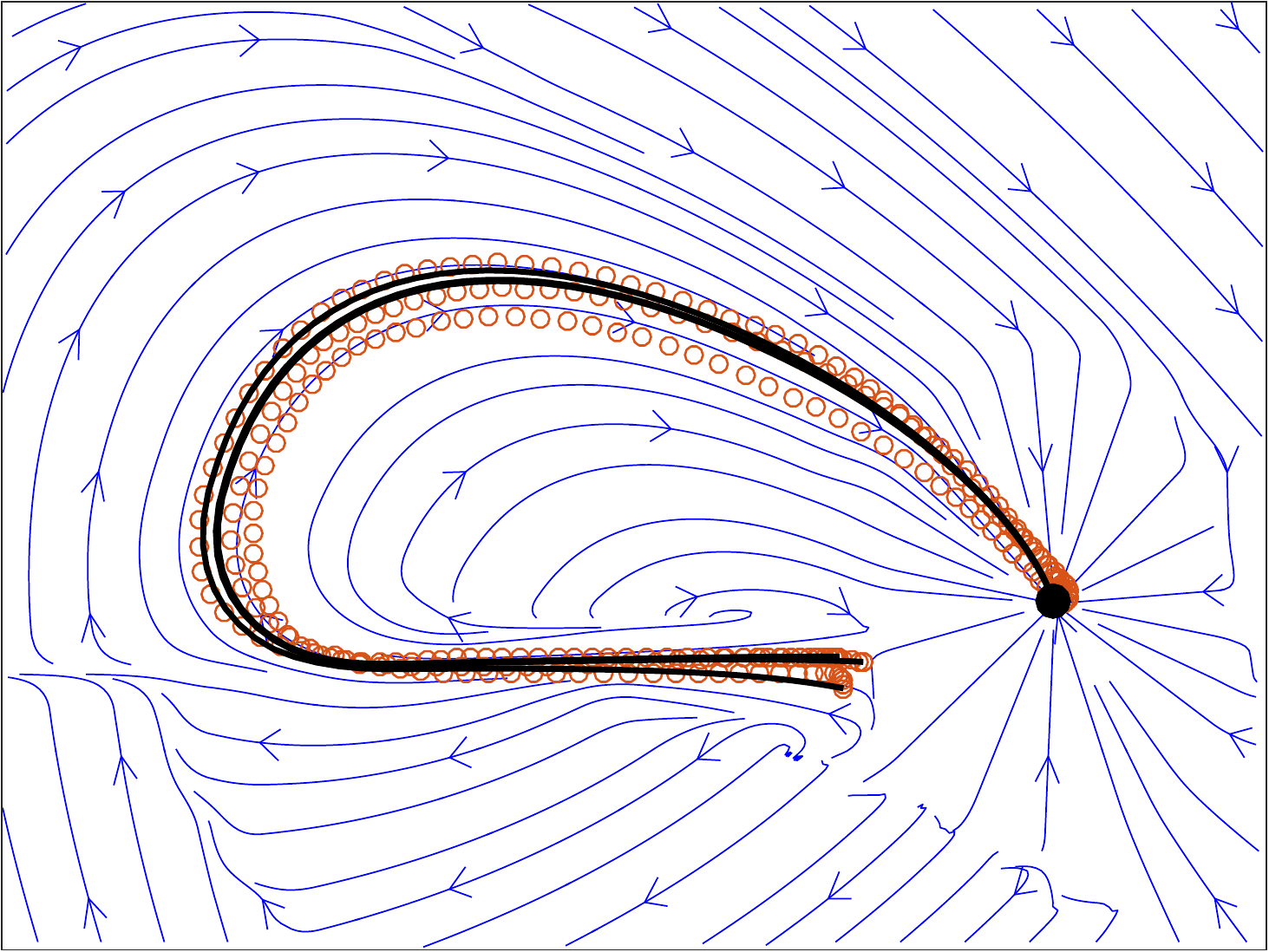}
	\includegraphics[width=0.07\textwidth]{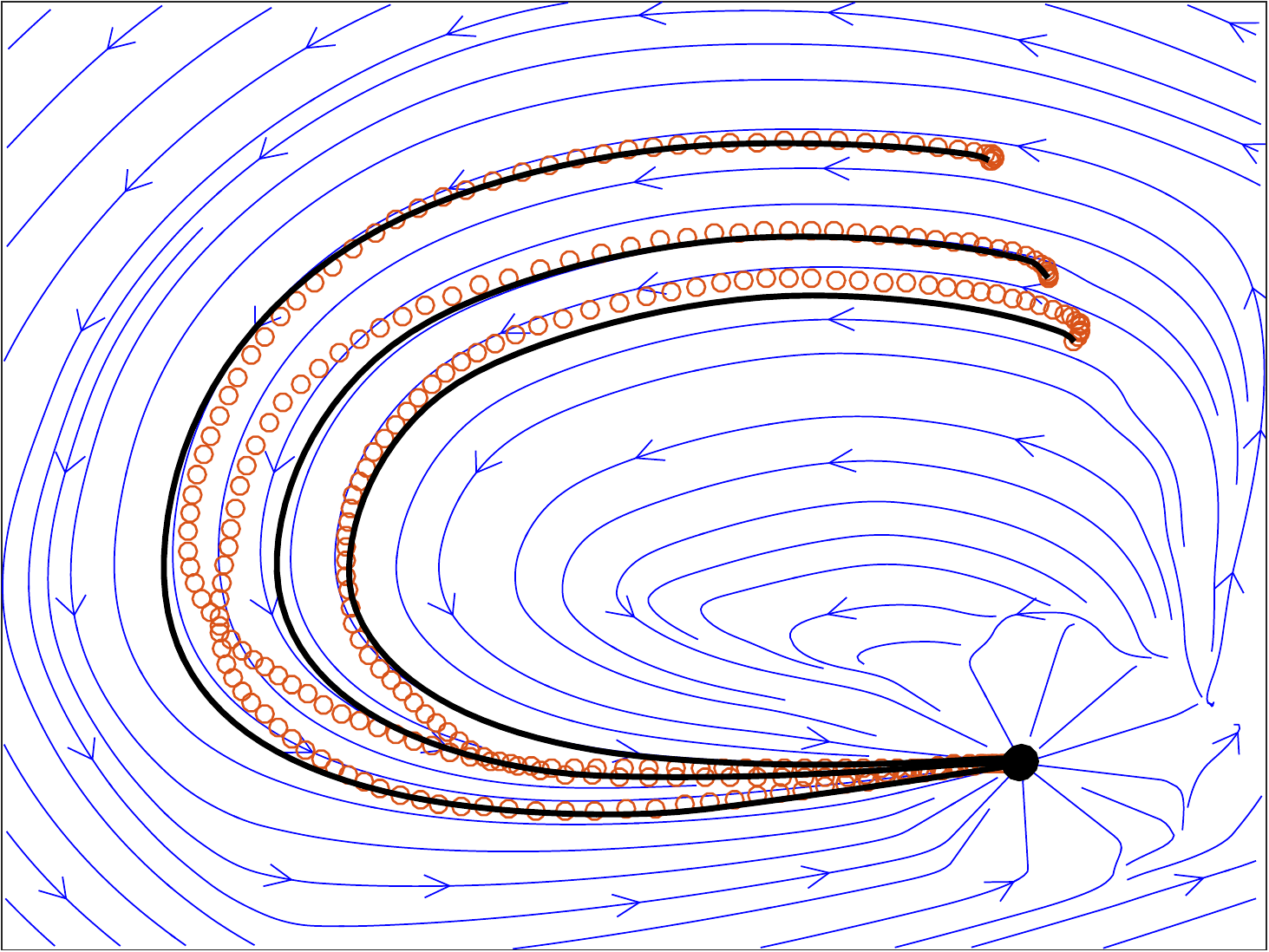}
	\includegraphics[width=0.07\textwidth]{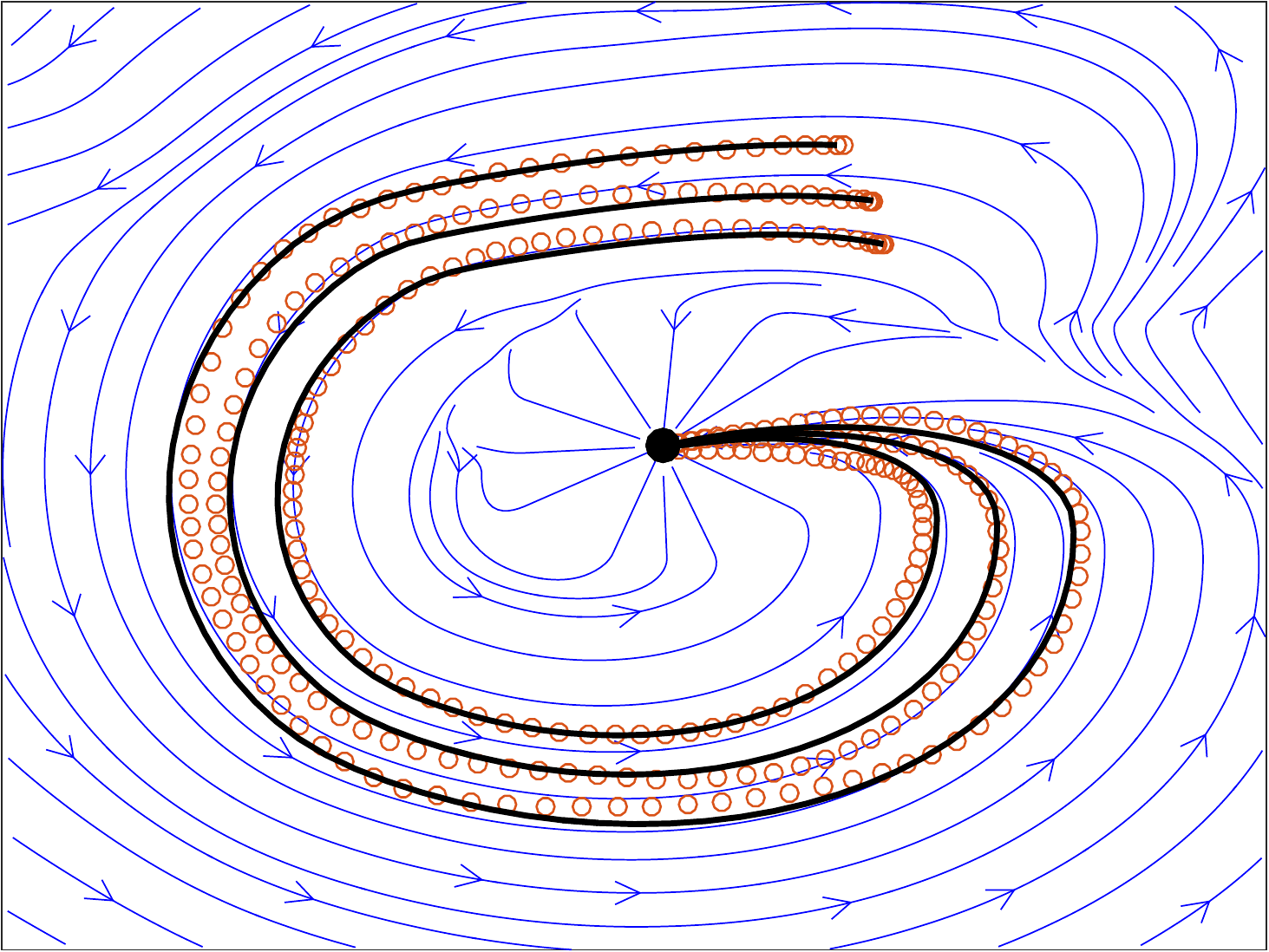}
	\includegraphics[width=0.07\textwidth]{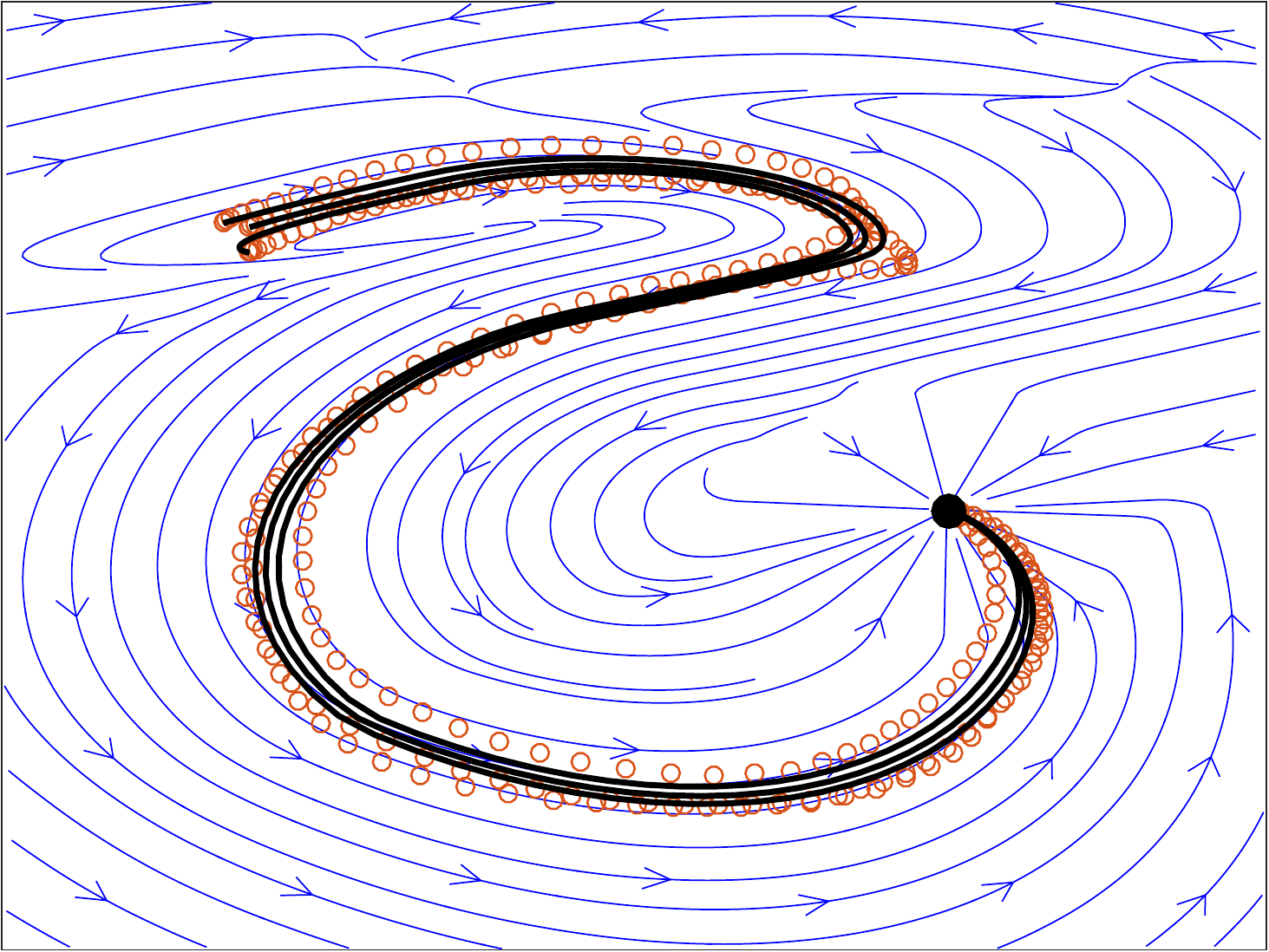}
  \includegraphics[width=0.07\textwidth]{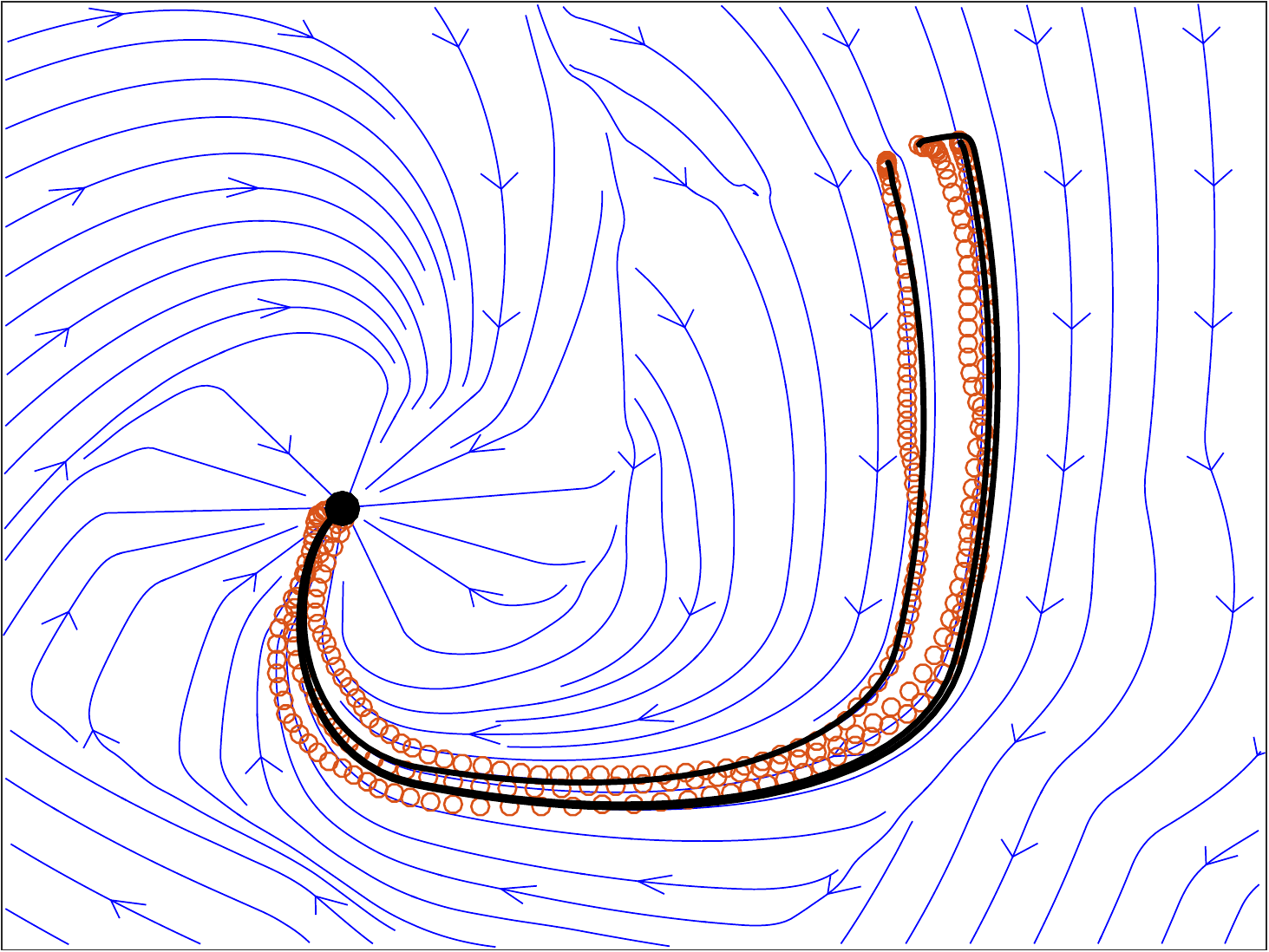}
	\includegraphics[width=0.07\textwidth]{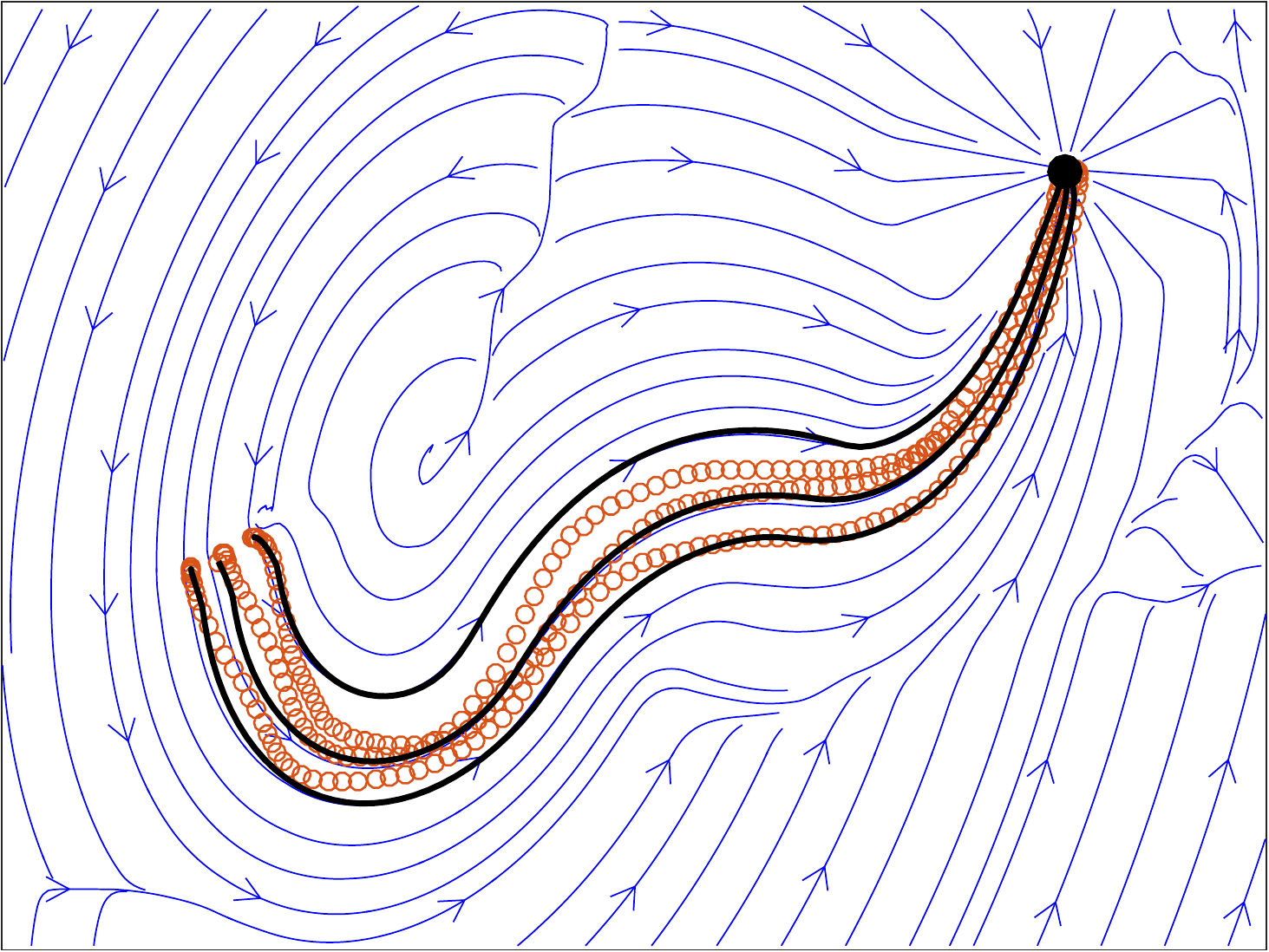}
	\includegraphics[width=0.07\textwidth]{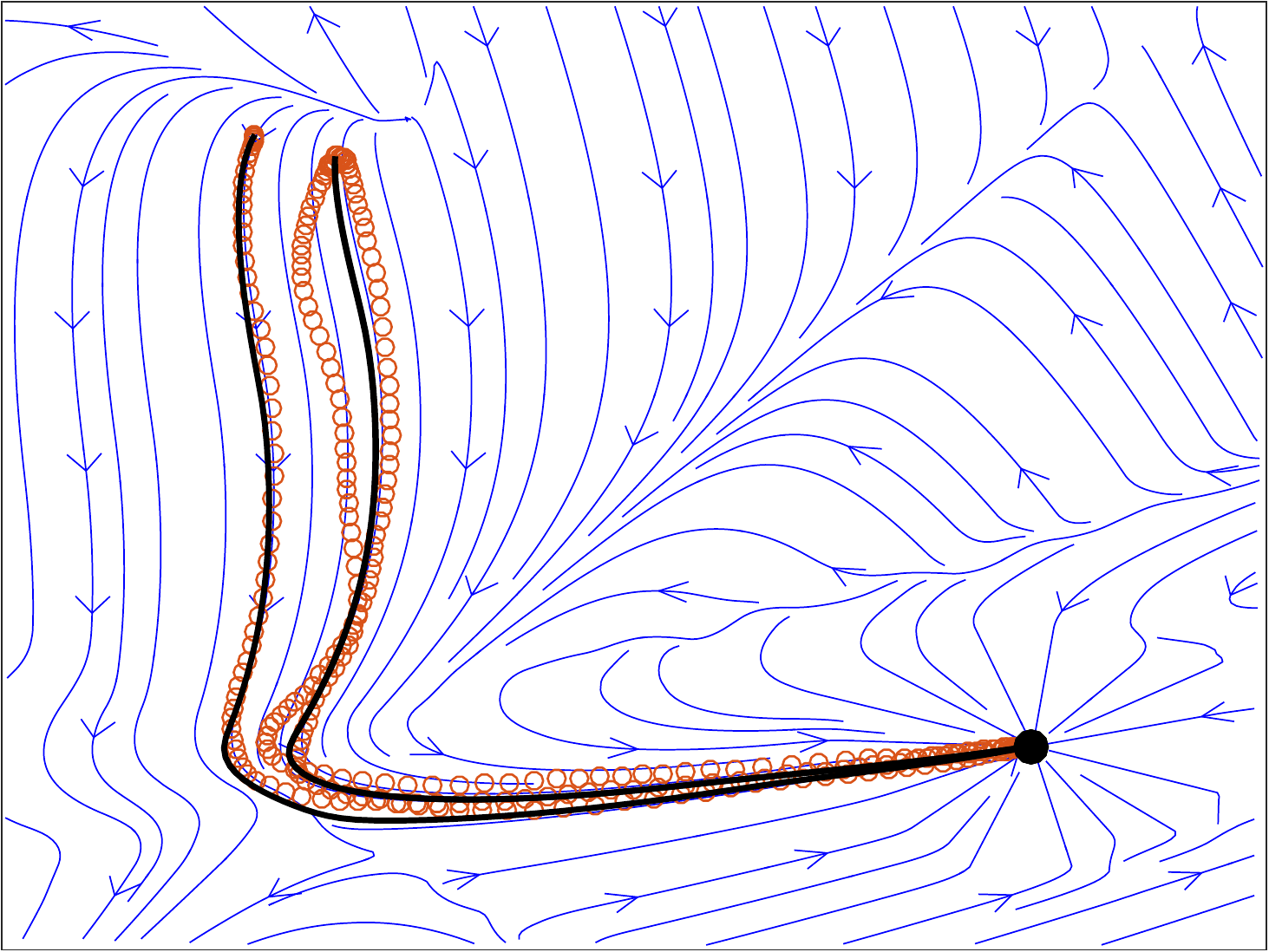}
	\includegraphics[width=0.07\textwidth]{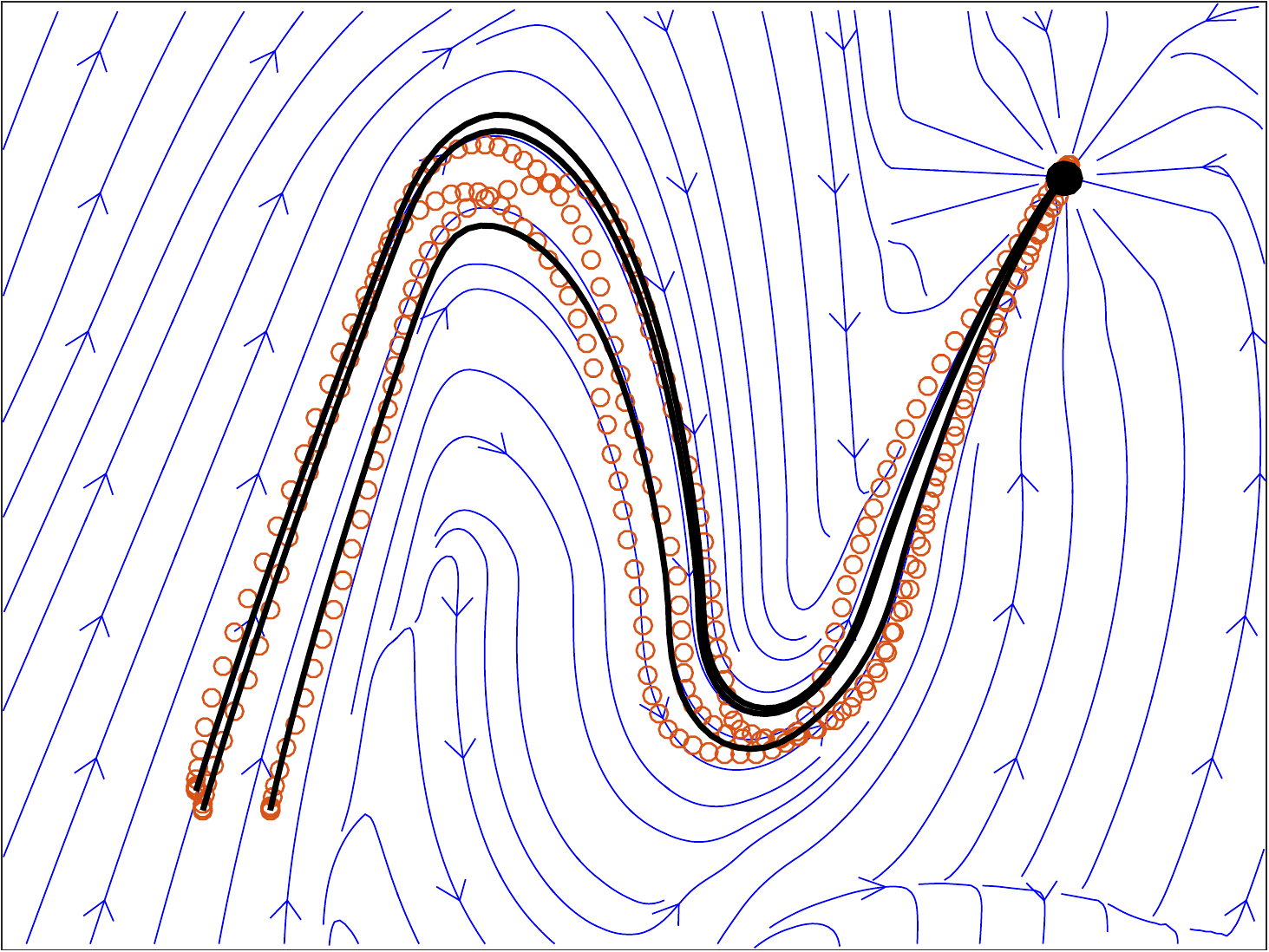}
	\includegraphics[width=0.07\textwidth]{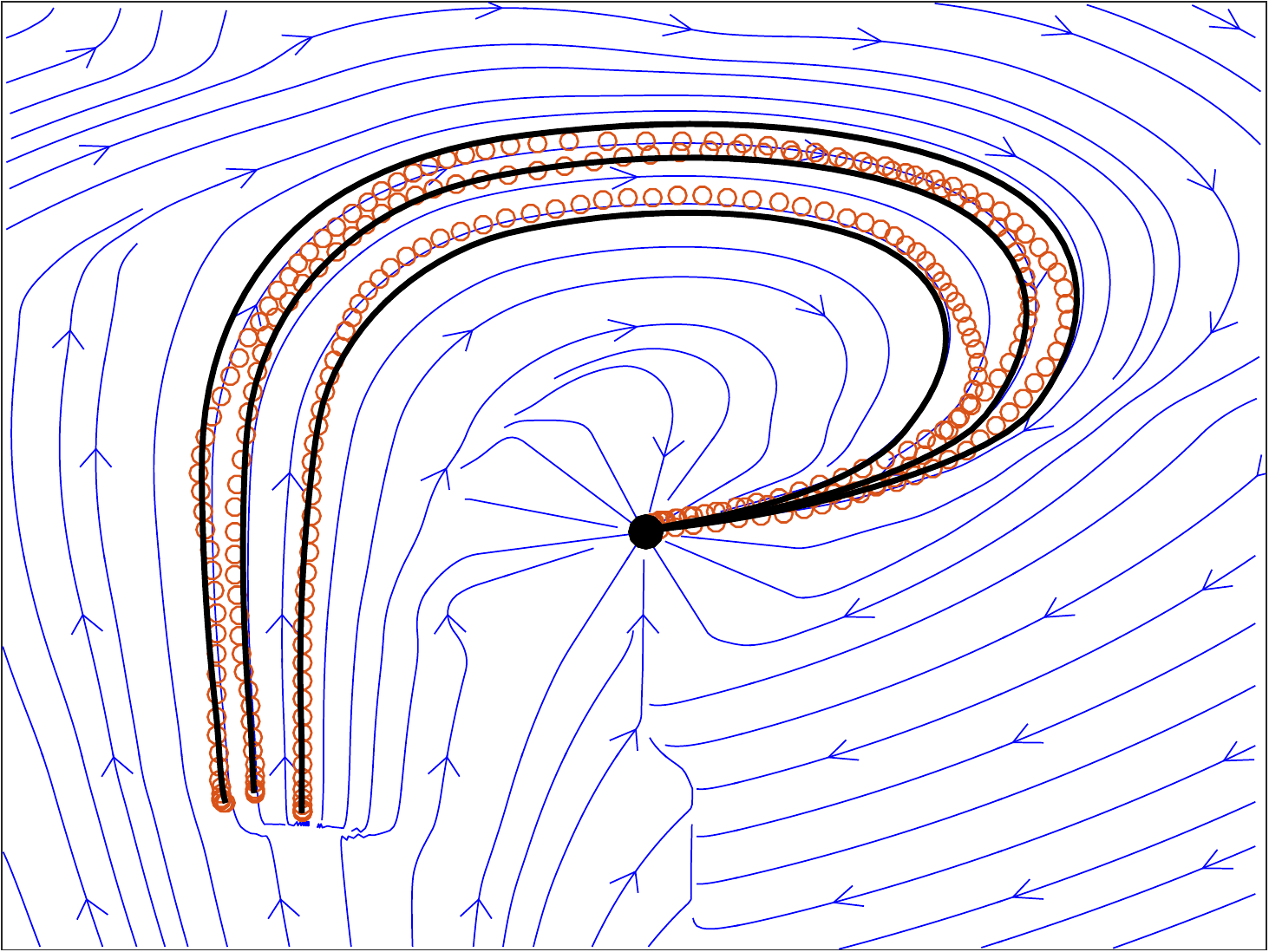}
	\includegraphics[width=0.07\textwidth]{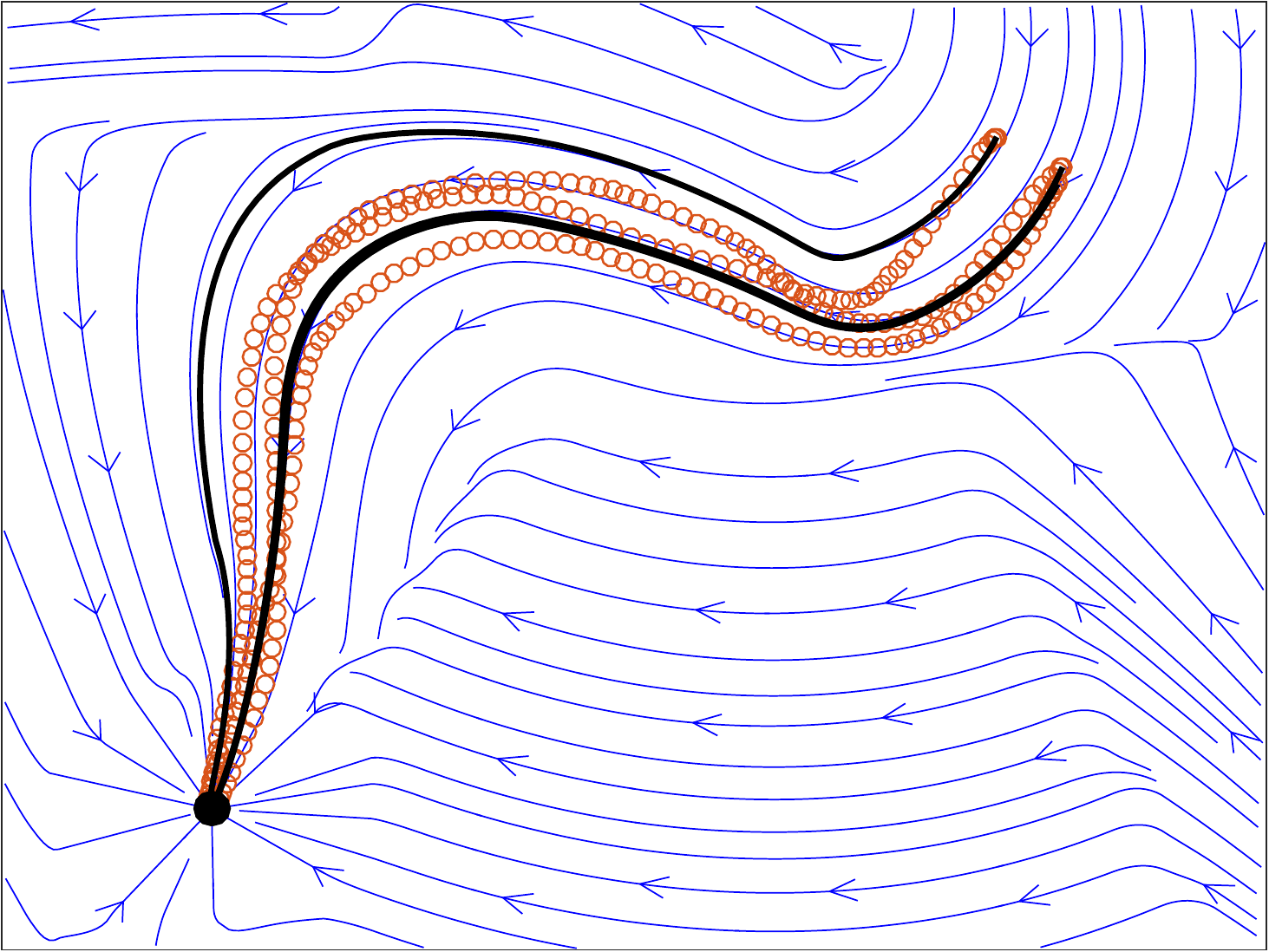}
	\includegraphics[width=0.07\textwidth]{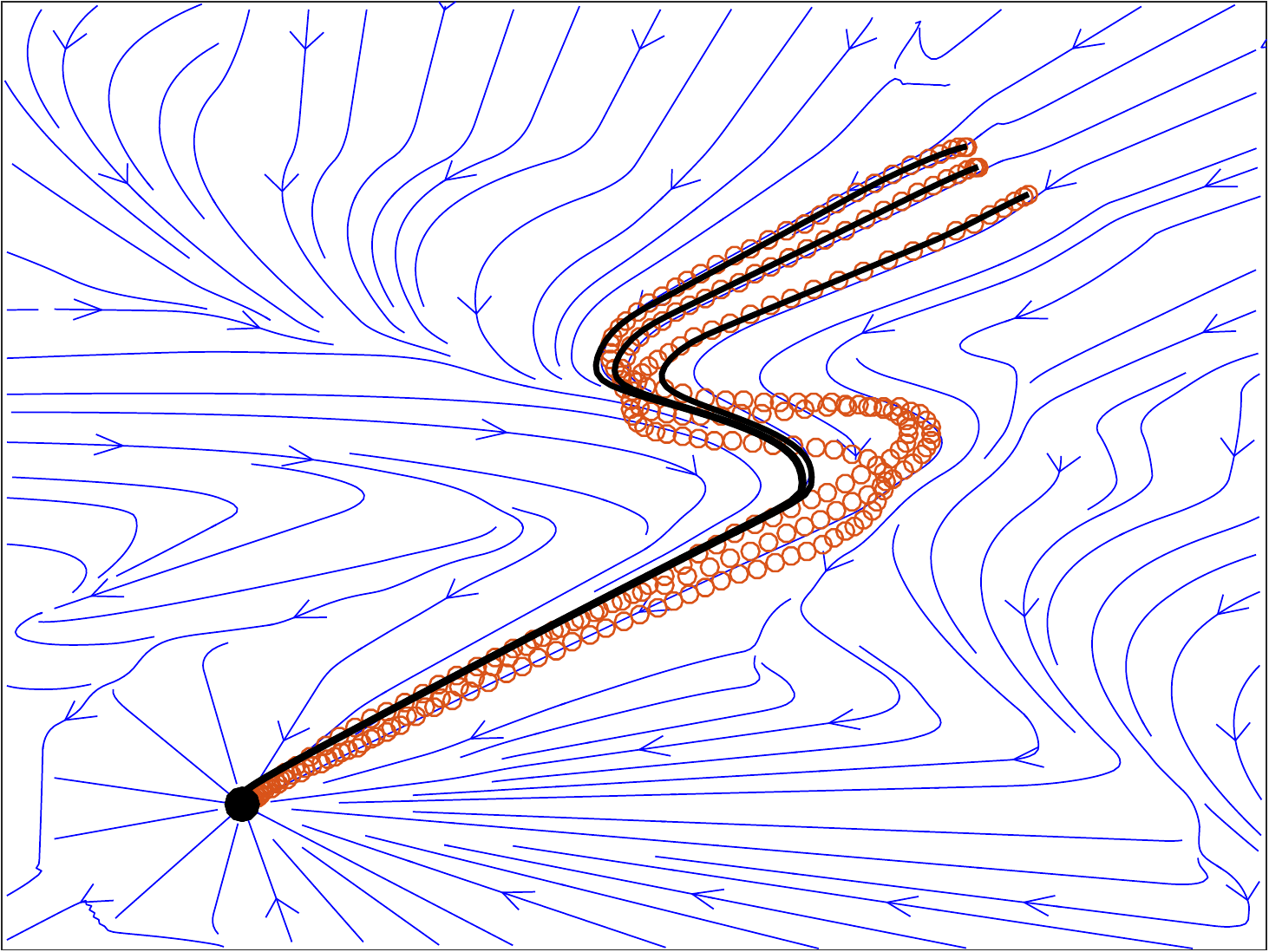}
	\includegraphics[width=0.07\textwidth]{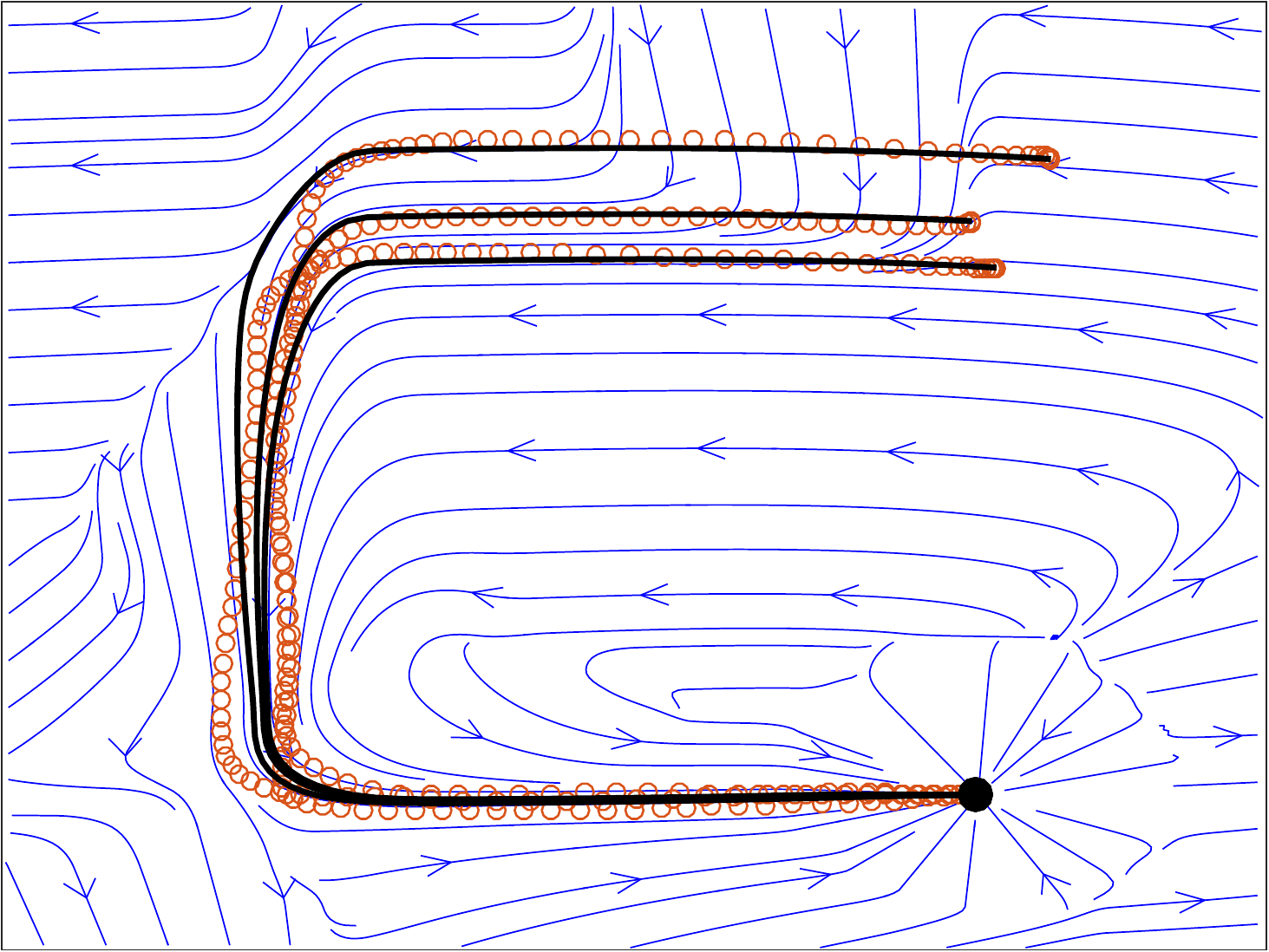}
	
	\includegraphics[width=0.07\textwidth]{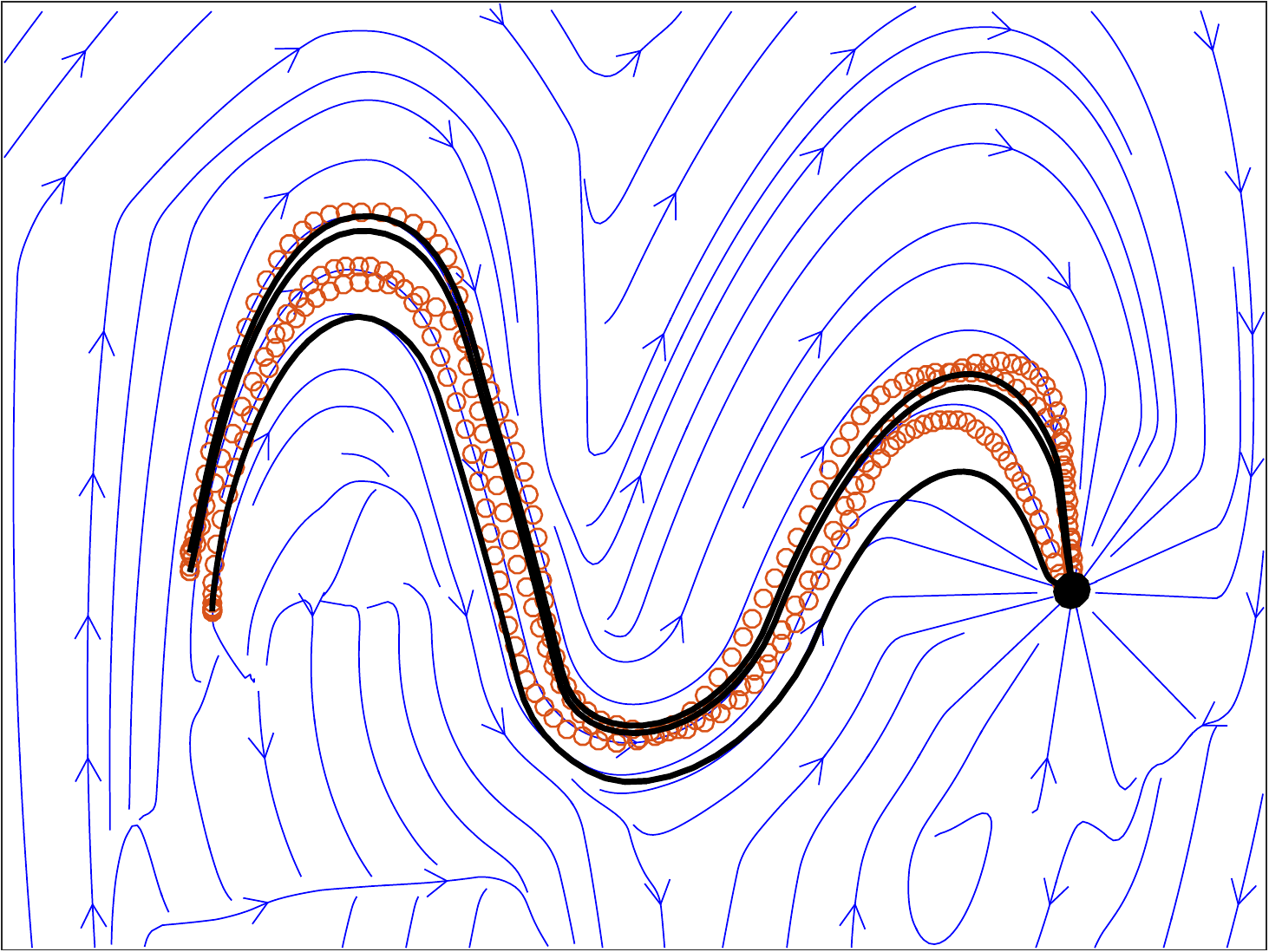}
	\includegraphics[width=0.07\textwidth]{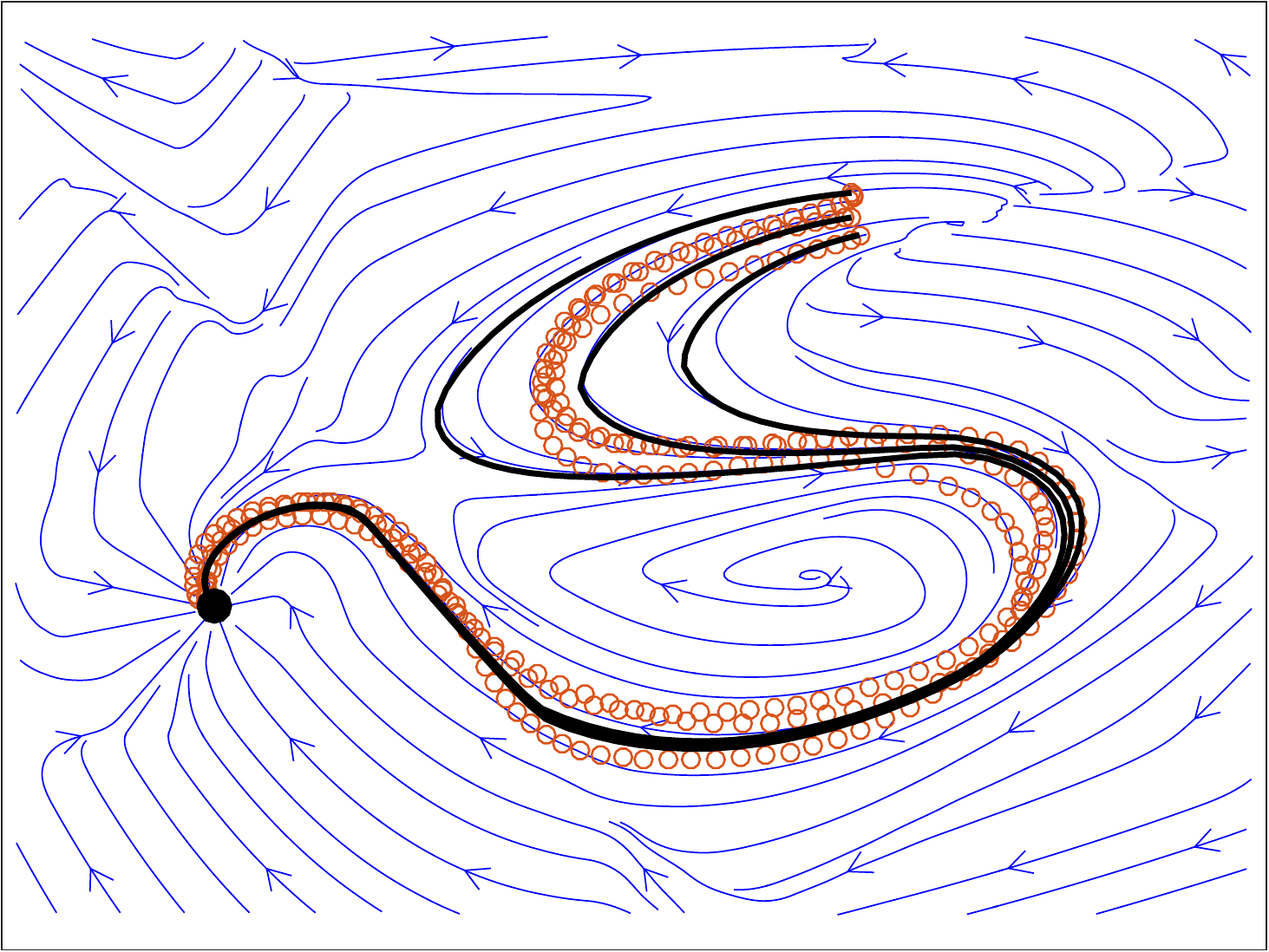}
	\includegraphics[width=0.07\textwidth]{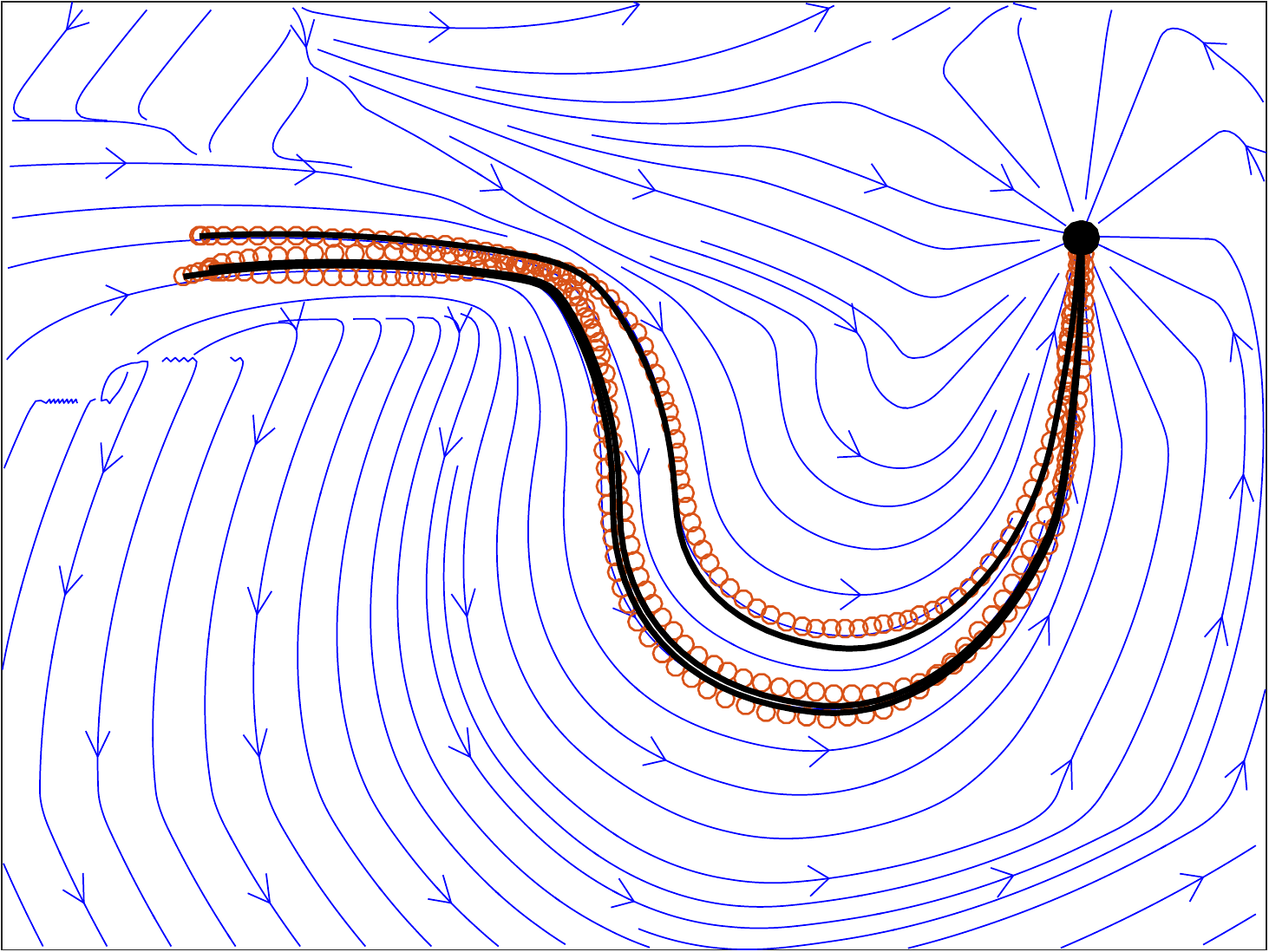}
	\includegraphics[width=0.07\textwidth]{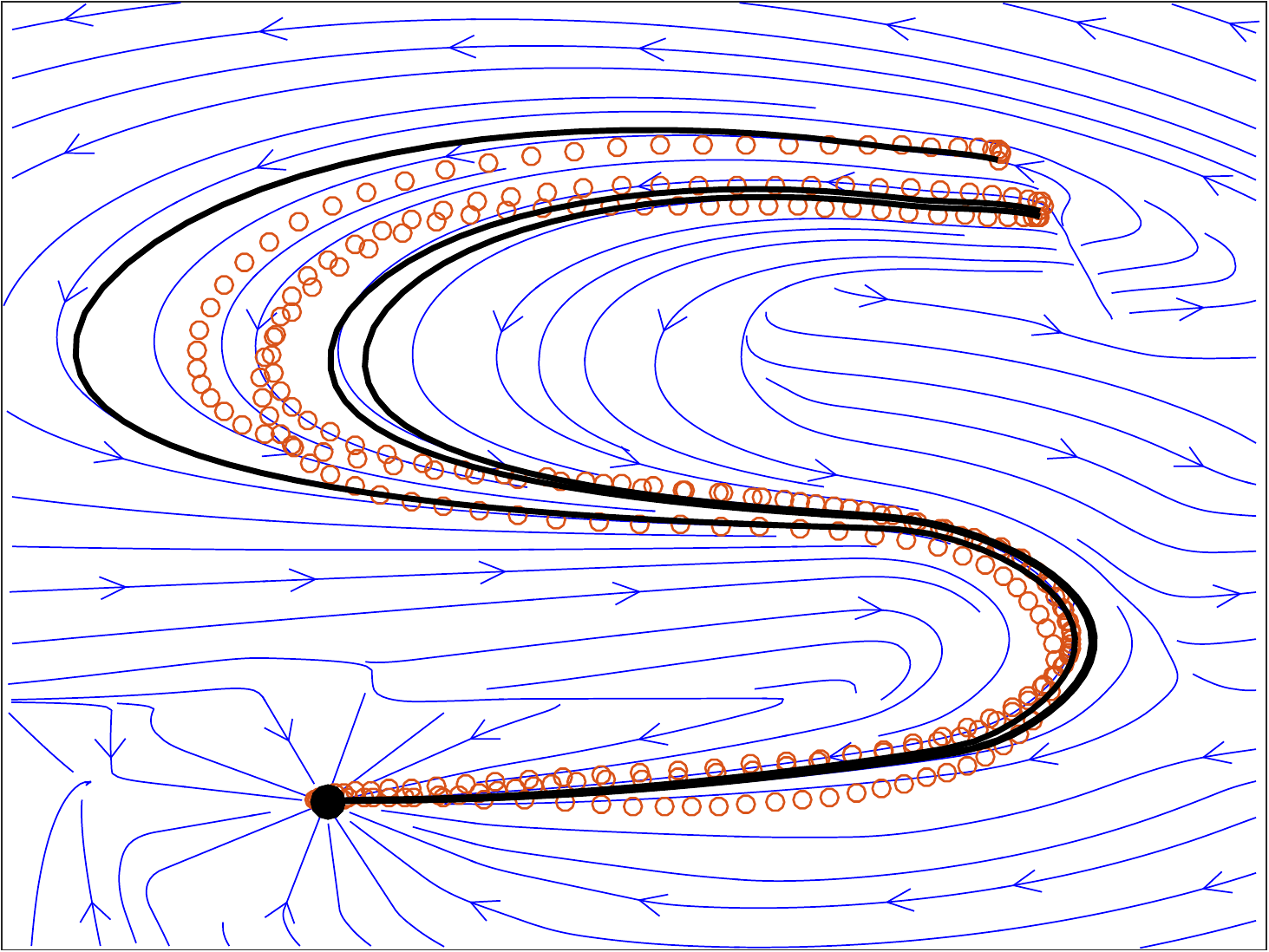}
	\includegraphics[width=0.07\textwidth]{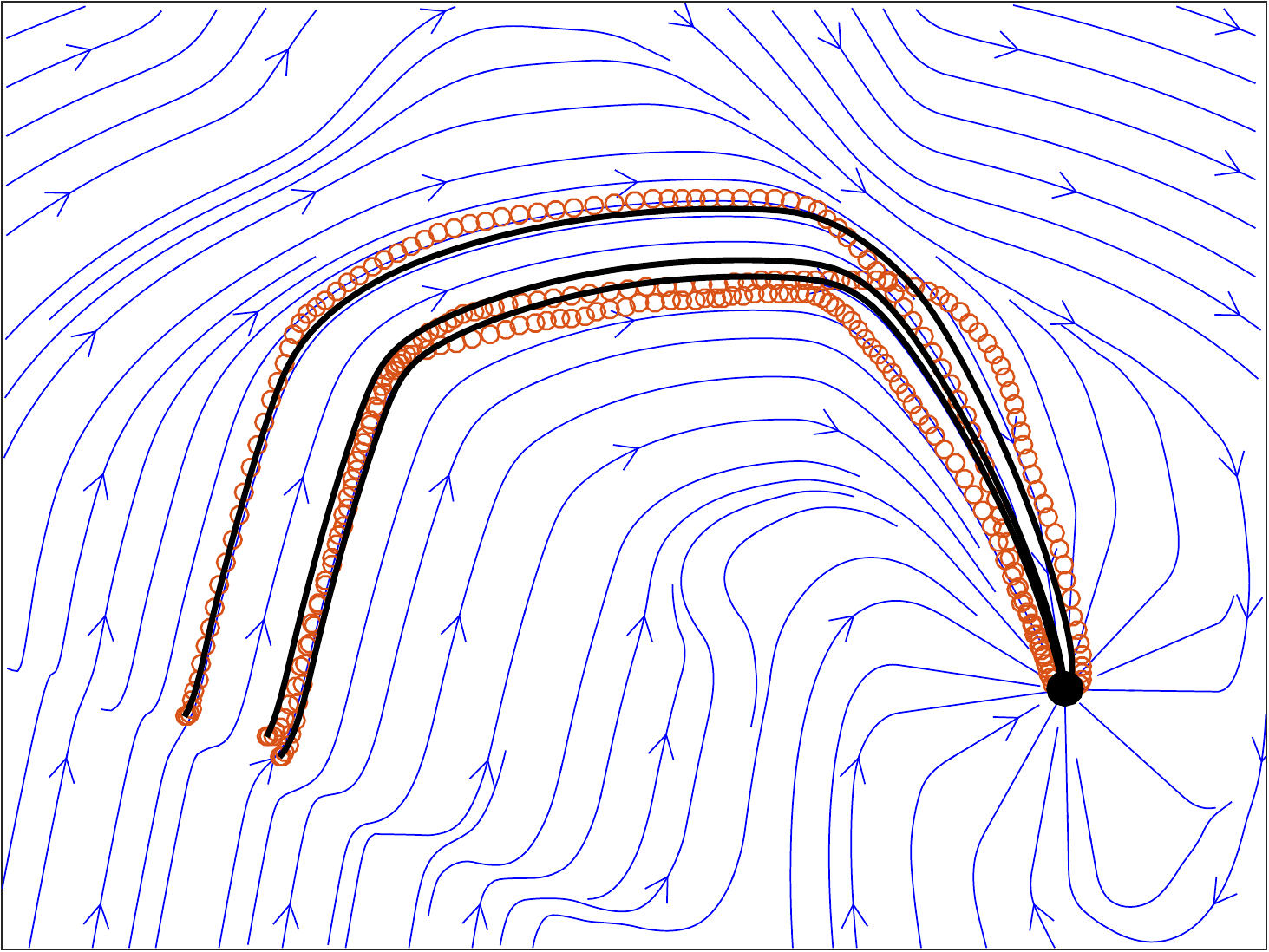}	
	\includegraphics[width=0.07\textwidth]{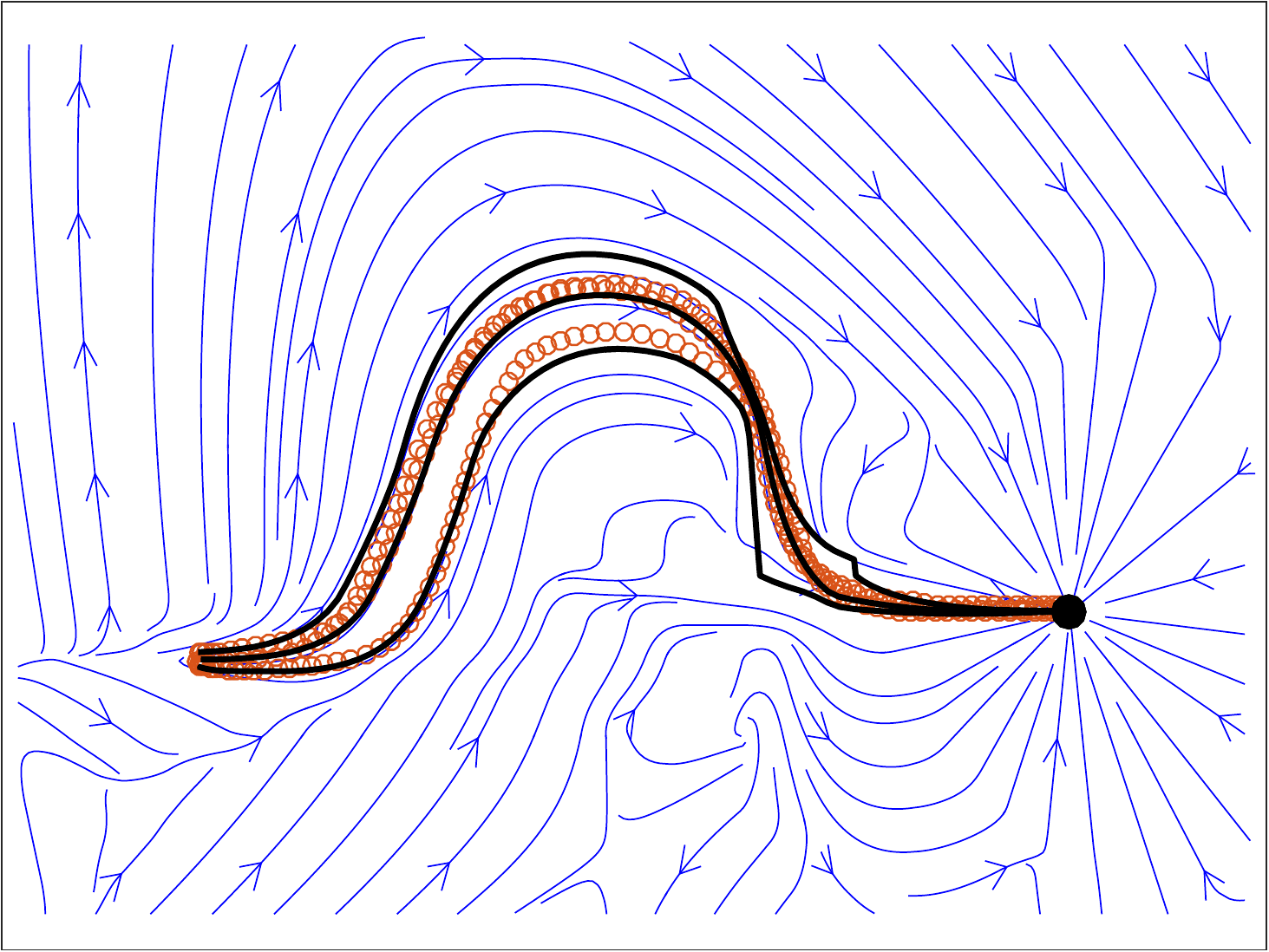}
	\includegraphics[width=0.07\textwidth]{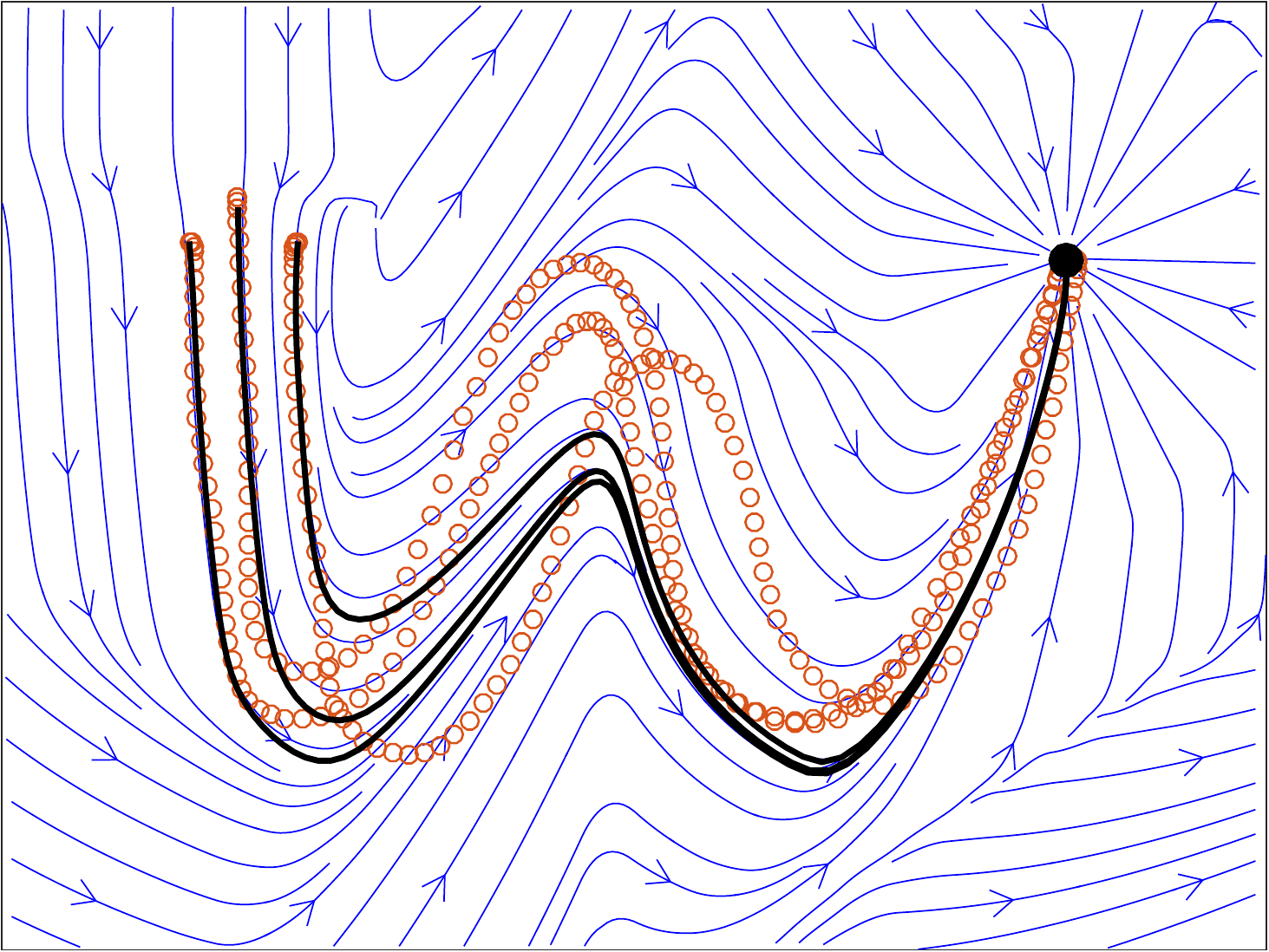}	
	\includegraphics[width=0.07\textwidth]{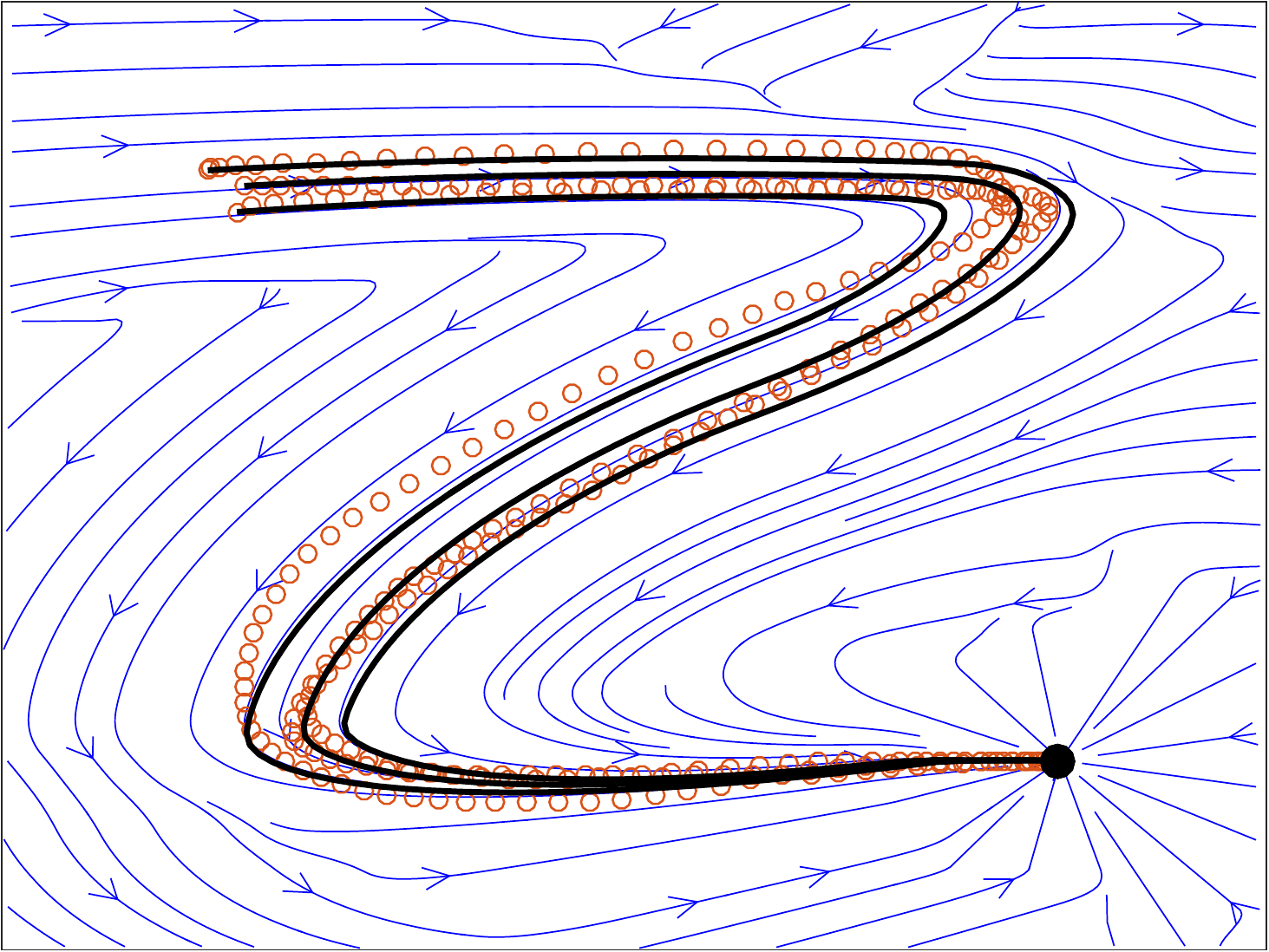}
	\includegraphics[width=0.07\textwidth]{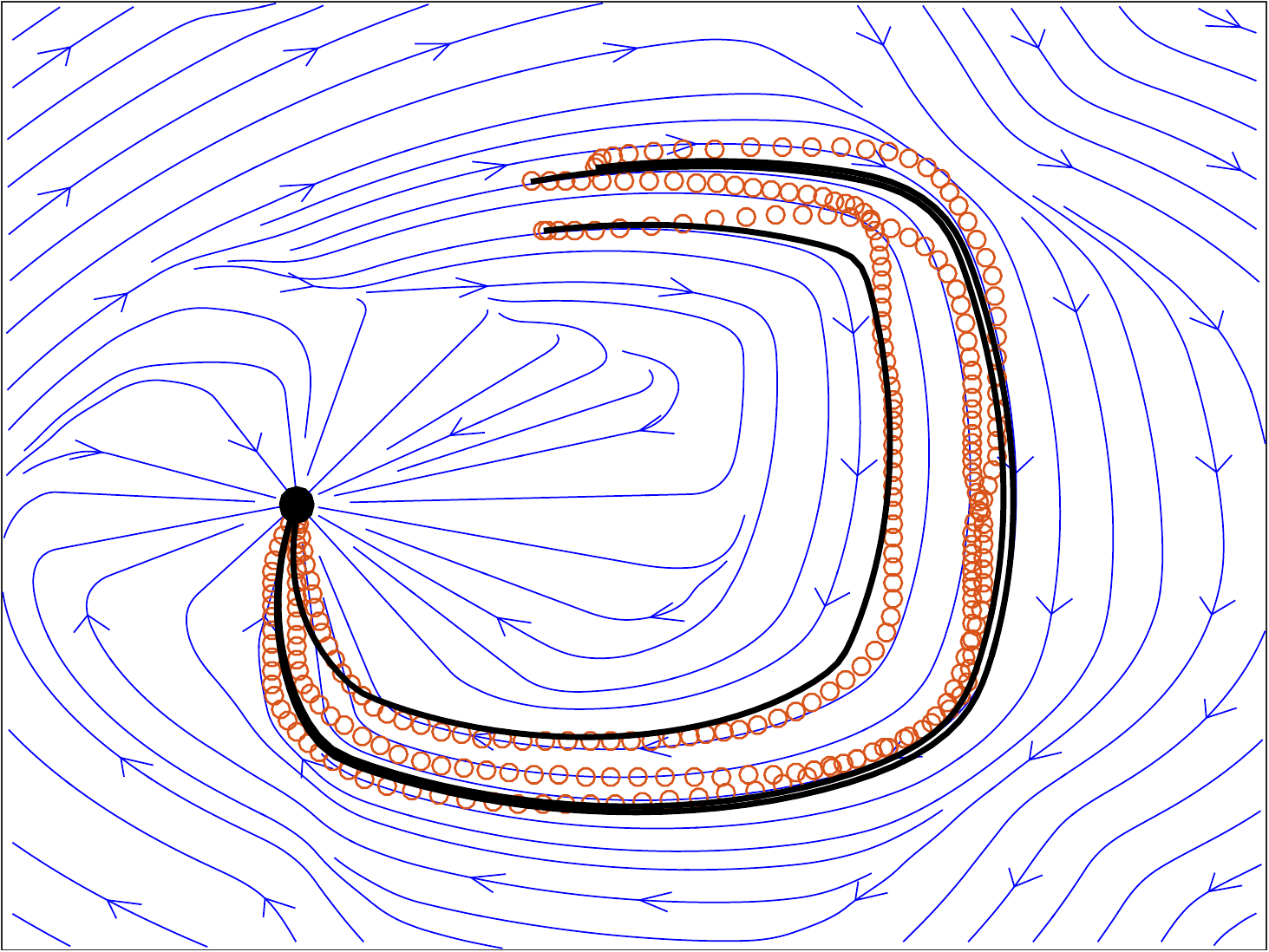}
	\includegraphics[width=0.07\textwidth]{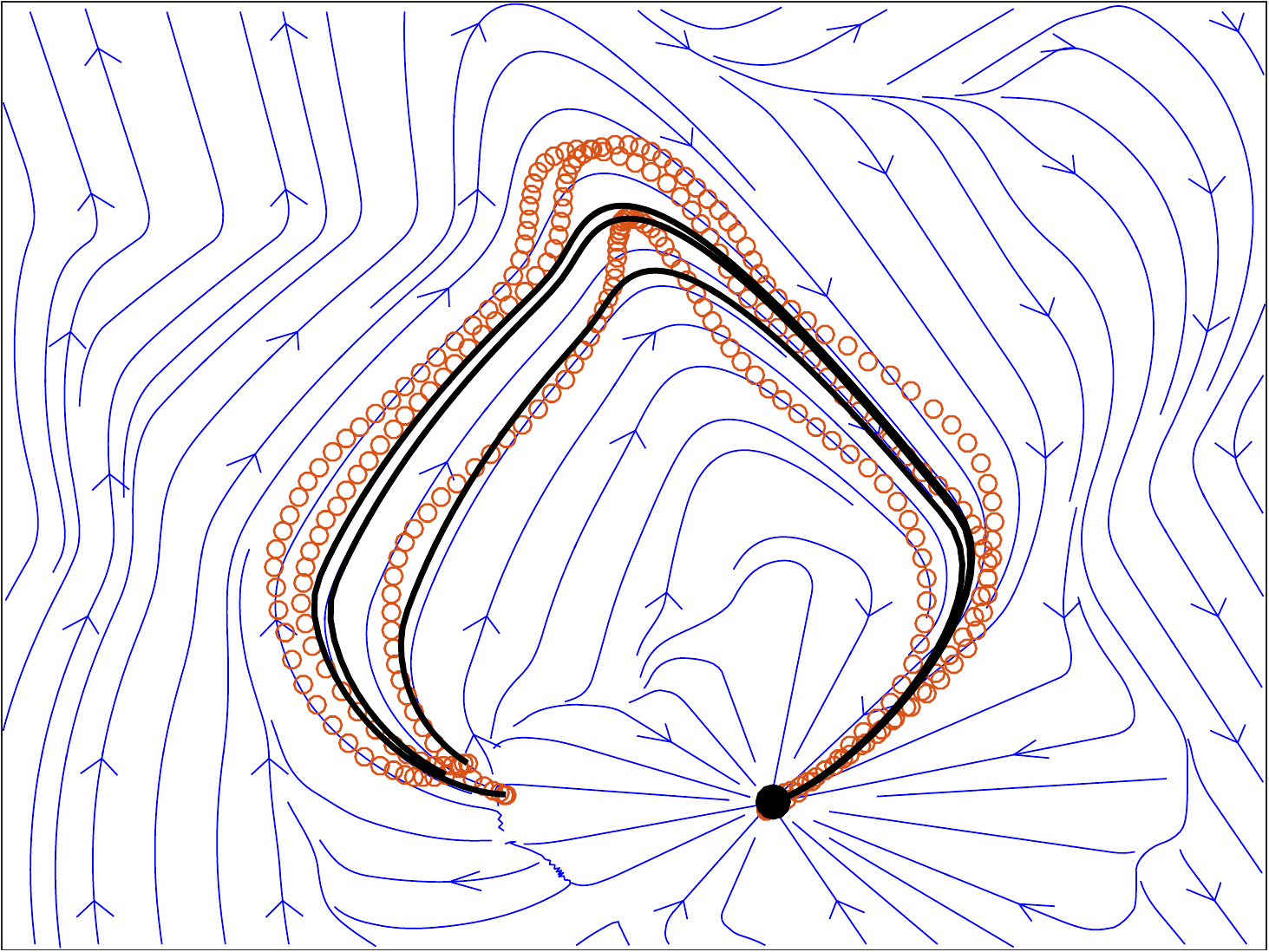}
	\includegraphics[width=0.07\textwidth]{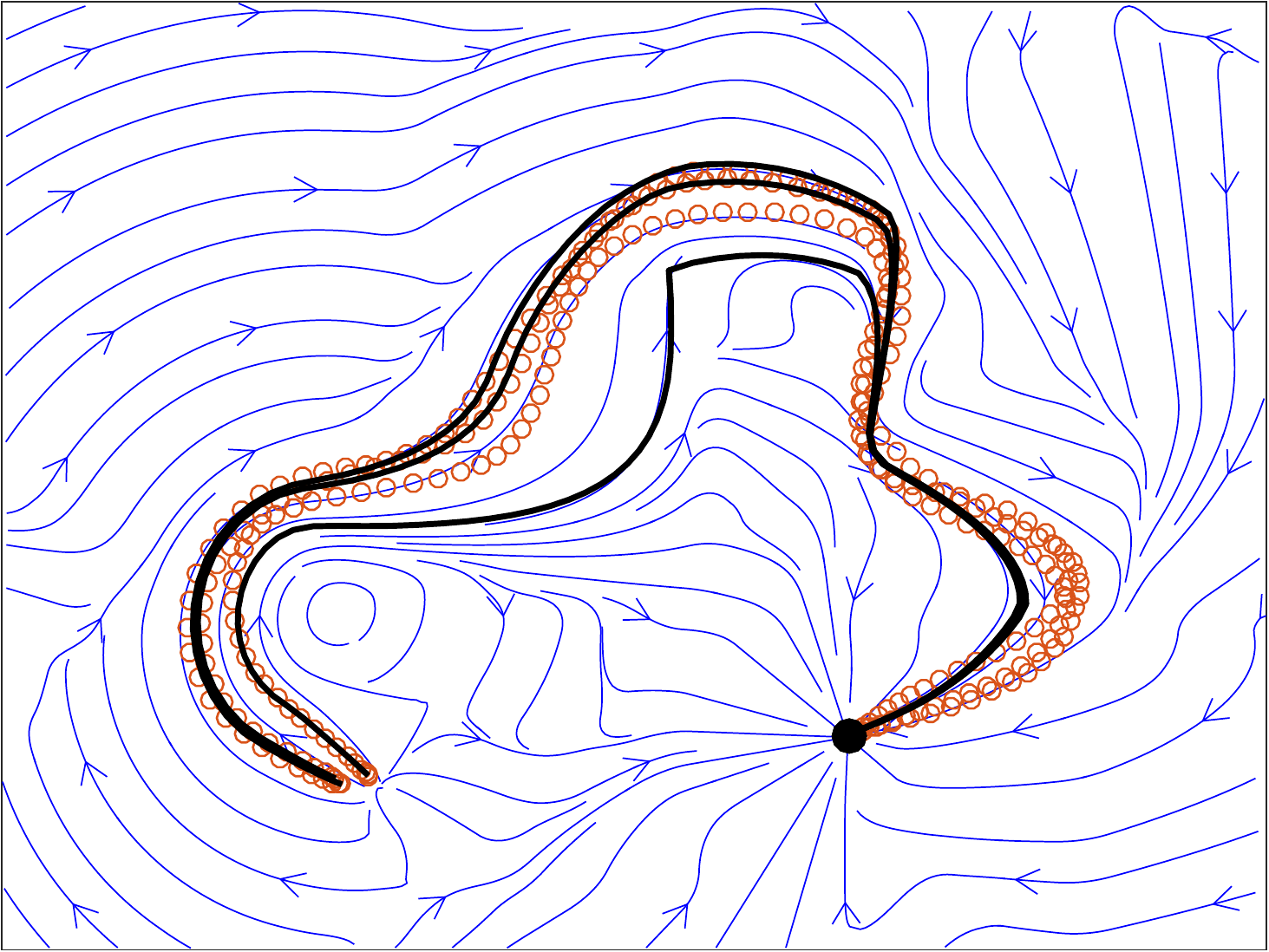}
	\includegraphics[width=0.07\textwidth]{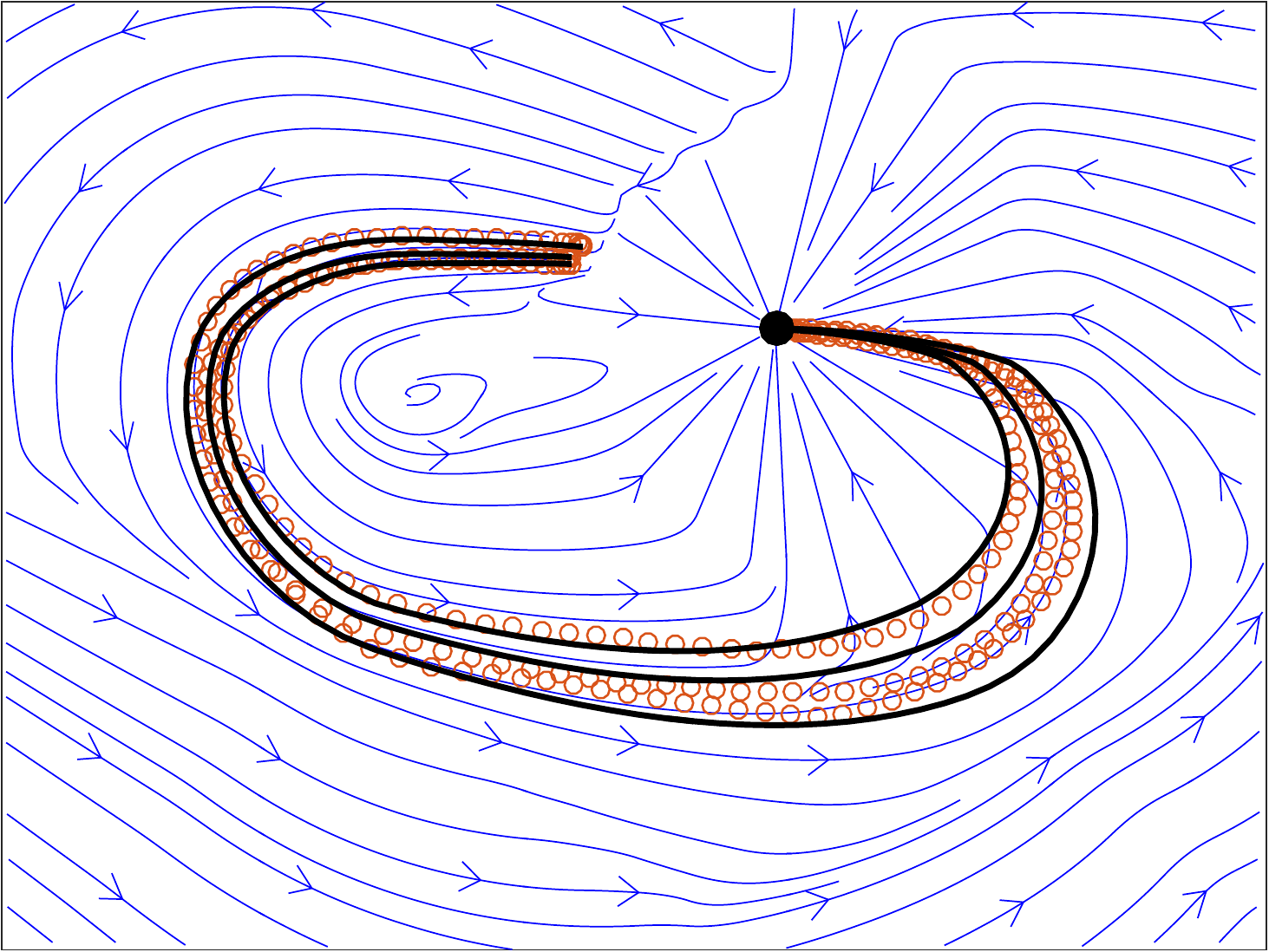}
	\includegraphics[width=0.07\textwidth]{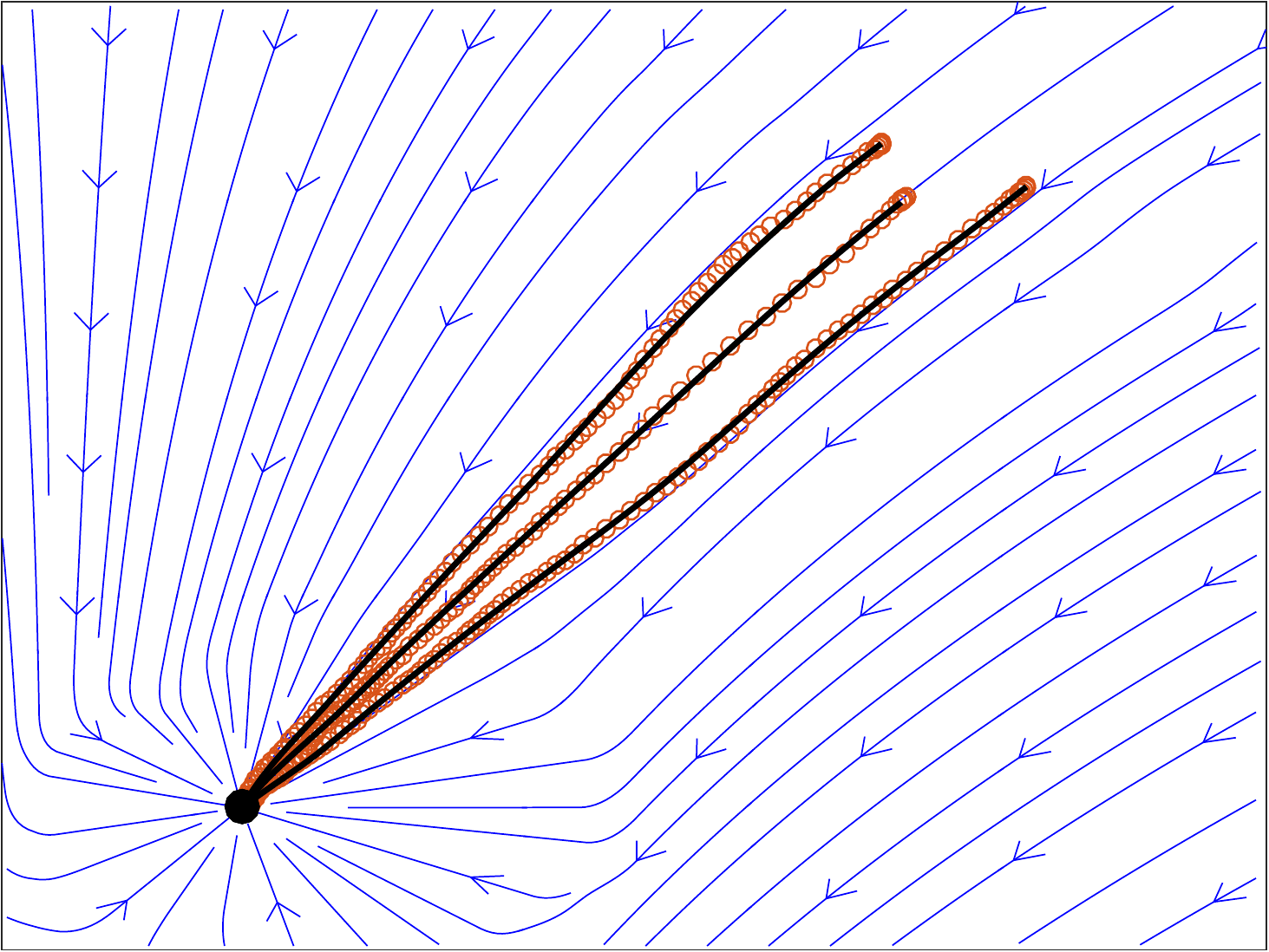}
    \caption{Qualitative results obtained with our approach on the $26$ complex motions of the LASA dataset. Brown circles represent the demonstrated positions, black solid lines the retrieved trajectories, and blue solid lines the streamlines of the dynamical system.}
    	\label{fig:lasa_qualitative}
\end{figure*} 
The results in this section show the capabilities of ESDS. First, we compare ESDS with prominent approaches in the field using the LASA Handwriting dataset \cite{LASA_dataset} as a benchmark. The dataset contains the $26$ stable 2D motions shown in Fig.~\ref{fig:lasa_qualitative}. Then, we present an experiment on a real robot.

\subsection{Accuracy and computational complexity}\label{subsec:exp_1}
The aim of this test is to quantitatively compare several approaches for motion generation with stable DS. The comparison is carried out considering two major quantities, namely the reproduction accuracy and the required training time. The accuracy is quantitatively measured using the swept error area ($SEA$) \cite{Clf} and the velocity root mean square error $V_{rmse}$ \cite{NeuralLearn2}. The $SEA$ is defined as $SEA = \sum_{d=1}^{D}\sum_{t=1}^{T-1}\mathcal{A}(\bfx_e^{t}, \bfx_e^{t+1}, \bfx_d^{t}, \bfx_d^{t+1})$, where $\bfx_{e}^{t}$ is the DS position and $\bfx_{d}^{t}$ the demonstrated position at time $t$ and demonstration $D$, $T$ is the number of samples in each demonstration, $D$ is the number of demonstrations, and $\mathcal{A}(\cdot)$ is the area of the tetragon formed by $\bfx_e^{t}$, $\bfx_e^{t+1}$, $\bfx_d^{t}$, and $\bfx_d^{t+1}$. Notice that DS trajectories are equidistantly re-sampled to contain exactly $T$ points. The $V_{rmse}$ is defined as $V_{rmse} = \sum_{d=1}^{D}\sqrt{\frac{1}{T}\sum_{t=1}^{T}\Vert \dot{\bfx}_{d}^t - \dot{\bfx}_{DS}(\bfx_{d}^t)\Vert^2}$, where $\dot{\bfx}_{DS}$ is the velocity retrieved from the learned DS. For our approach, $\dot{\bfx}_{DS}$ is defined as in \eqref{eq:paramet_ds_gamma}. 
In other words, the $SEA$ measures the error in reproducing the shape of the demonstrated motions, while the $V_{rmse}$ measures how the DS preserves the demonstrated velocities. 

Having defined the accuracy metrics, we present comparative results obtained by downsampling each demonstration in the LASA dataset to $100$ points and by considering the first three demonstrations for each motion. Qualitative results of ELDS on this dataset are depicted in Fig.~\ref{fig:lasa_qualitative}, while quantitative results are summarized in Tab.~\ref{tab:comparison_lasa}. For ESDS, the initial value of the storage is automatically determined as detailed in Sec.~\ref{subsec:storage_init}, while other quantities are defined as in Tab.~\ref{tab:smooth_functions} following the implementation in \cite{Kronander16,Smooth_functions}. In order to be consistent with state-of-the-art approaches \cite{SEDS,Blocher17,tau-SEDS}, we used GMR as regression technique in this evaluation also for ESDS and CLF-DM in \cite{Clf}. For all the approaches that uses GMR, the optimal number of Gaussian components $K^*$ is independently found for each approach and each motion by fitting a stable DS for each $K=4,\ldots,7$, computing the $SEA$ for each $K$, and choosing $K^*$ as the number of components that results in the minimum $SEA$. It is worth reminding that we cannot use GMR for the FDM approach in \cite{Perrin16} since it applies to linear DS, and that ESDS and CLF-DM can work with any regression technique (an example for ESDS is shown in Sec.~\ref{subsec:robot_experiments}). For FDM we use $10$ kernel functions and run the diffeomorphic matching algorithm for a maximum of $150$ iterations. C-GMR parameters are selected as in \cite{Blocher17}.
SEDS, CLF-DM, and $\tau$-SEDS require non-linear optimization loops. For SEDS, we use the \textit{likelihood} as cost function and run the optimization for a maximum of $500$ iterations. CLF-DM parameterizes the Lyapunov function as a weighted sum of asymmetric quadratic functions (WSAQF), whose parameters are learned via numerical optimization. We search for the optimal number of asymmetric components $L^*$ by repeating the training for $L=0,\ldots,4$ and choosing as $L^*$ the number of asymmetric components that results in the minimum $SEA$. Being significantly faster then SEDS, we let the optimization run for a maximum of $1500$ iterations. $\tau$-SEDS learns a WSAQF from demonstrations ($1500$ iterations), computes a diffeomorphism from the learned WSAQF, and applies SEDS ($500$ iterations) on the diffeomorphed data. 

Results obtained with the best set-up ($K^*$ and, when needed, $L^*$) are reported in Tab.~\ref{tab:comparison_lasa}. Looking at the results, it is clear that SEDS introduces strong deformations in the learned motions and it requires a relatively long training time. These results are in accordance with previous studies \cite{Clf, tau-SEDS,Blocher17}. FDM is extremely fast, but it introduces deformations in the learned motion. The reason is that learning from multiple demonstrations is not possible and one has to average across multiple demonstrations to get a unique trajectory. This is also the reason why the training time of FDM is almost constant across different motions. $\tau$-SEDS has a slightly better accuracy than FDM but is the slowest among the considered approaches. The long training time is expected since  $\tau$-SEDS needs to learn both a diffeomorphism (via CLF-DM) and a stable DS (via SEDS). Notice that the accuracy of $\tau$-SEDS can be improved by increasing the number of iterations in the SEDS algorithm. However, our implementation takes more than $3$ hours to find $K^*$ and $L^*$ for all the considered motions and generate the results. Therefore, we decided to limit the number of iterations to $500$. C-GMR performs well in terms of accuracy and training time. The problem with this approach is that the shape of the motion is significantly deformed outside a region of the state-space that contains the demonstrations. In other words, accuracy depends on a careful choice of the parameters that define this demonstration area. CLF-DM achieves high accuracy and it is the fastest among the approaches that exploit constrained optimization, i.e. SEDS and $\tau$-SEDS. The proposed ESDS does not introduce an additional training time or significant motion deviations. Therefore, with GMR regression, our approach has a training time comparable to C-GMR and FDM, it has the best $SEA$ value and a $V_{rmse}$ close to C-GMR and CLF-DM. These results show that ESDS is a good compromise between accuracy and training time. 


\begin{table}[t]
    \centering
    \caption{Reproduction error and training time of different approaches on the LASA dataset.} 
    \label{tab:comparison_lasa}
    \resizebox{\columnwidth}{!}{%
    {\renewcommand\arraystretch{1.3} 
	\begin{tabular}{ |c|c|c|c| }
	\hline
	Learning   & $SEA$ [mm\textsuperscript{2}] & $V_{rmse}$ [mm/s] & Train. Time [s]\\
	Method & (mean / range)  & (mean / range) & (mean / range)\\
	\hline
	\hline
	ESDS (Ours) &   431.5 / [26.0-1307] & 15.8 / [6.2-31.6] & 0.08 / [0.03-0.17] \\
	\hline
	SEDS \cite{SEDS} &  1022 / [34.9-2300] &  48.4 / [11.4-136.8] & 16.3 / [1.2-55.1] \\
	\hline
	CLF-DM \cite{Clf} &  460.7 / [16.6-1269] & 14.1 / [5.6-23.8] & 2.3 / [0.09-21.5] \\
	\hline
	$\tau$-SEDS \cite{tau-SEDS} & 537.0 / [26.4-1139] & 27.1 / [9.5-53.8] & 25.3 / [7.6-55.4] \\
	\hline
	C-GMR \cite{Blocher17} &  496.7 / [20.3-1840] & 14.0 / [6.2-24.1] & 0.1 / [0.03-0.28] \\
	\hline
	FDM \cite{Perrin16} &  550.5 / [42.0-1769] & 31.4 / [9.7-70.2] & 0.08 / [0.07-0.09] \\
	\hline
	\multicolumn{4}{ c }{*The source code of SEDS and CLF-DM is available on-line \cite{SEDS_Code,CLF_Code}.}\\
	\multicolumn{4}{ c }{**The author would like to thank N. Perrin and P. Schlehuber-Caissier for providing}\\
	\multicolumn{4}{ c }{ the source code of the FDM approach in \cite{Perrin16}.}
\end{tabular}
}}
\end{table}	



\subsection{Robot experiment}\label{subsec:robot_experiments}
\begin{figure}[t]
	\centering
	\subfigure[Demonstrations and retrieved trajectories.]{\includegraphics[width=\columnwidth]{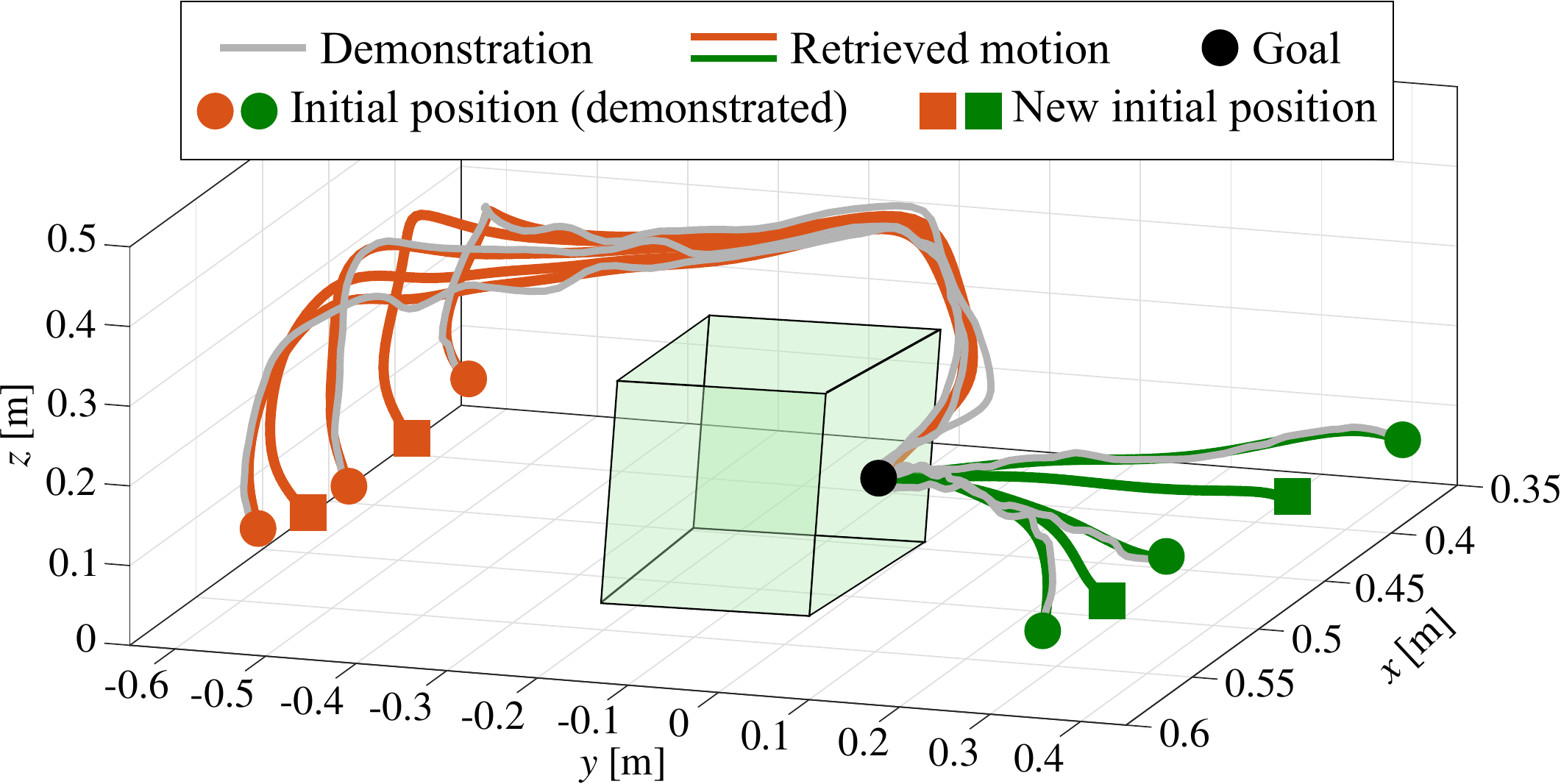}}
	\subfigure[Snapshots of the motion execution.]{\includegraphics[width=\columnwidth]{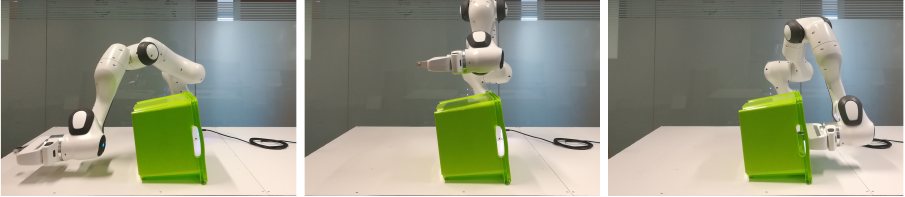}}
    	\caption{Results obtained by applying ESDS to retrieve stable motions on a real robot.}
    	\label{fig:experiment}
\end{figure}
The goal of this experiment is two-fold. First, we aim at showing that ESDS works with different regression techniques. Therefore, Gaussian process regression \cite{Rasmussen06} is used to retrieve $\bff(\bfx)$. Second, we show that a DS can encode different motions in different regions of the state-space. To this end, we provide $6$ demonstrations of a task consisting in entering into a box from different initial positions (Fig.~\ref{fig:experiment}). Demonstrations are provided by kinesthetically guiding the robot towards the task execution. In our setup, the robot's interface applies a low-pass filter with cut-off frequency of $100\,$Hz to the measured Cartesian positions (grey lines in Fig.~\ref{fig:experiment}(a)).   
Depending on the starting point, the robot has to surround the box to prevent a collision (brown lines in Fig.~\ref{fig:experiment}(a)) or it can follow an almost linear path (green lines in Fig.~\ref{fig:experiment}(a)). Both the behaviors are successfully encoded into the same DS, and ESDS is used to retrieve a stable motion converging to a unique target. As shown in Fig.~\ref{fig:experiment}(a), the robot is able to generalize the learned skill starting from initial positions that are not in the given training set. Snapshots of the surrounding motion are shown in Fig.~\ref{fig:experiment}(b). In this experiment, ESDS generates a motion trajectory by integrating the robot initial position and the robot tracks this reference trajectory using the build-in impedance controller. In principle, it is possible to combine ESDS and the approach in \cite{Kronander16} to feed the measured robot state during the execution. However, the combination may potentially affect the overall accuracy. Further investigations on this point are beyond the scope of this paper and are left as a future work. 

\section{Conclusions and Future Work}\label{sec:conclusion}
We presented the Energy-based Stabilizer of Dynamical Systems (ESDS), a novel approach for learning stable motions from human demonstrations. Inspired by energy tank-based controllers, we design a suitable Lyapunov function and stabilize the learned DS at run-time according to the virtual energy present into the system. The initial value of the energy, a parameter that significantly affects the reproduction accuracy, is also automatically estimated from training data. Comparisons on a public benchmark show that ESDS achieves high accuracy with reduced training time, while experiments on a real robot show that ESDS can stabilize a DS retrieved with different regression techniques.  

As a future work, we plan to provide more comprehensive results with experiments and comparisons in more challenging scenarios. Moreover, we plan to exploit ESDS in an incremental learning scenario where novel demonstrations of robotic skills are continuously provided at run-time. 





\bibliographystyle{IEEEtran}
\bibliography{bibliography}

\end{document}